\RequirePackage[loading]{tracefnt}

\documentclass[twoside]{article}

\usepackage[accepted]{aistats2024}
\usepackage[round]{natbib}

\usepackage{aistats2024}

\usepackage{mathrsfs}

\usepackage{tikz}
\usetikzlibrary{positioning}

\usepackage[utf8]{inputenc} 
\usepackage[T1]{fontenc}    
\usepackage{xr-hyper}
\usepackage{hyperref}       
\usepackage{url}            
\usepackage{booktabs}       
\usepackage{amsfonts}       
\usepackage{nicefrac}       
\usepackage{microtype}      
\usepackage{xcolor}         
\usepackage{bbold}
\usepackage{subcaption}

\usepackage{amsmath}
\usepackage{mathrsfs}
\usepackage{amsthm}
\usepackage{graphicx}
\usepackage{float}
\usepackage{mathtools}
\usepackage[ruled,vlined]{algorithm2e}

\newcommand{\mat}[1]{\ensuremath{\mathbf{#1}}}
\newcommand{\N}{\mathbb{N}}
\newcommand{\Z}{\mathbb{Z}}
\newcommand{\R}{\mathbb{R}}

\newcommand{\depth}{\mathrm{depth}}
\newcommand{\concat}{\ |\!|\  }
\DeclareMathOperator*{\argmax}{arg\,max}
\newcommand{\trees}{\mathscr{T}}
\newcommand{\str}{\mathrm{str}}
\newcommand{\opensymbol}{ [\![}
\newcommand{\closesymbol}{]\!]}

\renewcommand{\vec}[1]{\ensuremath{\mathbf{#1}}}
\newcommand{\vecs}[1]{\ensuremath{\mathbf{\boldsymbol{#1}}}}

\newcommand{\ten}[1]{\mat{\ensuremath{\boldsymbol{\mathcal{#1}}}}}
\newcommand{\ttm}[1]{\times_{#1}}
\newcommand{\Rbb}{\mathbb{R}}

\definecolor{violet}{rgb}{0.58, 0.0, 0.83}

\definecolor{teal}{rgb}{0.0, 0.5, 0.5}

\makeatletter
\newtheorem*{rep@theorem}{\rep@title}
\newcommand{\newreptheorem}[2]{%
\newenvironment{rep#1}[1]{%
 \def\rep@title{#2 \ref{##1}}%
 \begin{rep@theorem}}%
 {\end{rep@theorem}}}
\makeatother

\theoremstyle{plain}
\newtheorem{definition}{Definition}[section]

\newtheorem{lemma}{Lemma}

\newtheorem{theorem}{Theorem}
\newtheorem{corollary}{Corollary}
\newreptheorem{theorem}{Theorem}
\newcommand{\Prob}{\mathbb{P}}

\title{Simulating Weighted Automata over Sequences and Trees with Transformers}

\author{}
\begin{document}
\twocolumn[

\aistatstitle{Simulating Weighted Automata over Sequences and Trees with Transformers}

\aistatsauthor{Michael Rizvi \And Maude Lizaire \And  Clara Lacroce \And Guillaume Rabusseau}

\aistatsaddress{Mila \& DIRO \\ Université de Montréal \And Mila \& DIRO \\ Université de Montréal \And McGill University \\ Mila, Montréal, Canada \And  Mila \& DIRO \\ Université de Montréal, \\CIFAR AI Chair}
]

\begin{abstract}
Transformers are ubiquitous models in the natural language processing (NLP) community and have shown impressive empirical successes in the past few years. However, little is understood about how they reason and the limits of their computational capabilities. These models do not process data sequentially, and yet outperform sequential neural models such as RNNs. Recent work has shown that these models can compactly simulate the sequential reasoning abilities of deterministic finite automata (DFAs). This leads to the following question: can transformers simulate the reasoning of more complex finite state machines? In this work, we show that transformers can simulate weighted finite automata (WFAs), a class of models which subsumes DFAs, as well as weighted tree automata (WTA), a generalization of weighted automata to tree structured inputs. We prove these claims formally and provide upper bounds on the sizes of the transformer models needed as a function of the number of states the target automata. Empirically, we perform synthetic experiments showing that transformers are able to learn these compact solutions via standard gradient-based training.
\end{abstract}

\section{INTRODUCTION}
Transformers are the backbone of modern NLP systems \citep{transformers}. These models have shown impressive gains in the past few years. Large pretrained language models can translate texts, write code and can solve math problems, tasks which all require some level of sequential reasoning capabilities~\citep{brown2020language, chen2021evaluating}. However, unlike recurrent neural networks~(RNNs)~\citep{elman1990,Schmidhuber}, transformers do not perform their computations sequentially. Instead they process all input tokens in parallel. This defies our intuition as these models do not possess the inductive bias that naturally arises from treating a sequence from beginning to end.

In order to understand how transformers implement sequential reasoning, recent work by~\cite{liu2022transformers} studied connections between deterministic finite automata~(DFAs) and transformers. DFAs are simple models that perform deterministic sequential reasoning on strings from a given alphabet, which makes them a perfect candidate for exploring sequential reasoning in attention-based models. 
To do so, the authors consider the perspective of simulation. Informally, a transformer is said to simulate a DFA if for an input sequence of length $T$, it can output the sequence of states visited throughout the DFAs computation. The authors also consider how the complexity needed to perform this simulation task varies as a function of $T$. They find that it is possible to simulate all DFA at length $T$ with a transformer of size $O(\log T)$. They also show that in the case where the automaton is solvable, it is possible to achieve such a result with $O(1)$ size. This sheds light on how transformers can \textit{compactly} encode sequential behavior without explicitly performing sequential computation. 

However this is not representative of the capacities of transformers. The type of reasoning they implement can indeed go much further than the simple reasoning done by DFAs. In this work, we propose to go further and consider (i) weighted finite automata~(WFAs), a family of automata that generalize DFAs by computing a real-valued function over a sequence instead of simply accepting or rejecting it, and (ii) weighted tree automata~(WTA), a generalization of weighted automata to tree structured inputs. We show that transformers can simulate both WFAs as well as WTAs, and that they can do so \textit{compactly}.

More precisely, we show that, using hard attention and bilinear layers, transformers can \textit{exactly} simulate all WFAs at length $T$ with $\mathcal{O}(\log T)$ layers. Moreover, we show that using a more standard transformer implementation, with soft attention and an MLP (multilayer perceptron), transformers can \textit{approximately} simulate all WFAs at length $T$ up to arbitrary precision with $\mathcal{O}(\log T)$ layers and MLP width constant in $T$. This first set of results shows that transformers can learn shortcuts to sequence models significantly more complex than deterministic finite automata. 
Our second set of results is about computation over trees. For WTAs, the notion of simulation we introduce assumes that the transformer is fed a string representation of a tree and outputs the states of the WTA for each subtree of the input. We show that transformers can simulate WTA to arbitrary precision at length $T$ with $\mathcal O(\log T)$ layers over balanced trees. Since the class of WTA subsumes classical~(non-weighted) tree automata, an important corollary we obtain is that transformers can also simulate tree automata. 
Our results thus extend the ones of \cite{liu2022transformers} for DFAs in two directions: from boolean to real weights and from sequences to trees. 

Empirically, we study to which extent transformers can be trained to simulate WFAs. We first show that compact solutions can be found in practice using gradient-based optimization. To do so, we train transformers on simulation tasks using synthetic data. We also investigate if the number of layers and embedding size of such solutions scale as theory suggests. 

\section{PRELIMINARIES}
\subsection{Notation}
    We denote with $\N$, $\Z$ and $\R$ the set of natural, integers and real numbers, respectively. We use bold letters for vectors~(\textit{e.g.} $\vec{v} \in \R^{d_1}$), bold uppercase letters for matrices (\textit{e.g.} $\mat{M} \in \R^{d_1 \times d_2}$) and bold calligraphic letters for tensors~(\textit{e.g.} $\ten{T} \in \R^{d_1 \times \hdots \times d_n}$). All vectors considered are column vectors unless otherwise specified. We denote with $\mat{I}$ the identity matrix and write $\mat{I}_m$ to denote the $m \times m$ identity matrix.
    We will also denote $\mat{0}$ as the matrix full of zeros and use $\mat{0}_{m\times n}$ to denote the $m \times n$ zero matrix or simply $\mat{0}_m$ when said matrix is square.
    The $i$-th row and the $j$-th column of a matrix $\mat{M}$ are denoted by $\mat{M}_{i,:}$ and $\mat{M}_{:,j}$. We denote the Frobenius norm of a matrix as $\| \mat{M}\|_F$
    Finally, we will use $\mathbf{e}_i$ to refer to the $i$th canonical basis vector.
    
Let $\Sigma$ be a fixed finite alphabet of symbols, $\Sigma^*$ the set of all finite strings~(words) with symbols in $\Sigma$ and $\Sigma^n$ the set of all finite strings of length $n$. We use $\varepsilon$ to denote the empty string. Given $p,s \in \Sigma^*$, we denote with $ps$ their concatenation.
    
\subsection{Weighted Finite Automata}
    
    Weighted finite automata are a generalization of finite state machines \citep{droste,Mohri09,salomaa}. This class of models subsumes deterministic and non-deterministic finite automata, as WFAs can calculate a function over strings in addition to accepting or rejecting a word. While general WFAs can have weights in arbitrary semi-rings, we focus our attention on WFAs with real weights, as they are more relevant to machine learning applications. 
    
    \begin{definition}
    A \emph{weighted finite automaton} (WFA) of $n$ states over $\Sigma$ is a tuple $\mathcal{A} = \langle \boldsymbol{\alpha} , \{\mat{A}^\sigma\}_{\sigma \in \Sigma},  \boldsymbol{\beta} \rangle$, where $\boldsymbol{\alpha},$ $\boldsymbol{\beta} \in \R^n$ are the initial and final weight vectors, respectively, and $\mat{A}^\sigma \in \R^{n \times n}$ is the matrix containing the transition weights associated with each symbol $\sigma \in \Sigma$. Every WFA $\mathcal{A}$ with real weights realizes a function $f_A : \Sigma^* \to \R$, \emph{i.e.} given a string  $x = x_1 \cdots x_t \in \Sigma^*$, it returns $f_\mathcal{A}(x) = \boldsymbol{\alpha} ^\top \mat{A}^{x_1} \cdots \mat{A}^{x_t} \boldsymbol{\beta} = \boldsymbol{\alpha} ^\top \mat{A}^x \boldsymbol{\beta}$. 
    \end{definition}

    To simplify the notation, we write the product of transition maps $\mat{A}^{x_1} \cdots \mat{A}^{x_t}$ as $\mat{A}^{x_1\cdots x_t}$ or even $\mat{A}^{x_{1:t}}$ for longer sequences. 
    
    Similarly to ordinary DFAs, we define the \textit{state} of a WFA on a word $x_1\cdots x_t$ to be the product $\boldsymbol{\alpha^\top}\mathbf{A}^{x_1}\hdots \mathbf{A}^{x_t}$.  
    In light of this, we define $\mathcal{A}(x)$ to be the function that returns the sequence of states for a given word $x \in \Sigma^T$. More formally, we have
    \begin{align*}
    \mathcal{A}(x) = (\boldsymbol{\alpha}^\top, \boldsymbol{\alpha}^\top\mathbf{A}^{x_1}, 
    \boldsymbol{\alpha}^\top\mathbf{A}^{x_1x_2},
    \hdots, \boldsymbol{\alpha}^\top\mathbf{A}^{x_{1:T}})^\top
    \end{align*}
    There exist many model families encompassed by WFAs, one of the most well-known subsets of these models are hidden Markov models~(HMMs). 

\subsection{Weighted Tree Automata}\label{subsec:wta.prelim}

Weighted tree automata~(WTA) extend the notion of WFAs to the tree domain. In all generality, WTAs can be defined with weights over an arbitrary semi-ring and have as domain the set of ranked trees over an arbitrary ranked alphabet~\citep{droste}. Here, we only consider WTAs with real weights~(for their relevance to machine learning applications) defined over binary trees~(for simplicity of exposition).  We first formally define the domain of WTAs we will consider. 
\begin{definition}
    Given a finite alphabet $\Sigma$, the set of binary trees with leafs labeled by symbols in $\Sigma$ is denoted by $\trees_{\Sigma}$~(or simply $\trees$ if the leaf alphabet is clear from context). Formally, $\trees_{\Sigma}$ is the smallest set such that $\Sigma \subset \trees_{\Sigma}$ and 
         $(t_1,t_2) \in  \trees_{\Sigma}$ for all $t_1,t_2  \in \trees_{\Sigma}$.
\end{definition}

A WTA~(with real weights) computes a function mapping trees in $\trees_\Sigma$ to real values. In the context of machine learning, they can thus be thought of as parameterized models for functions defined over trees~(e.g. probability distributions or scoring functions). The computation of a WTA is performed in a bottom up fashion: (i) for each leaf, the state of a WTA with $n$ states is an $n$-dimensional vector, (ii) the states for all subtrees are computed recursively from the ground up by applying a bilinear map to  the left and right child of each internal nodes (iii) similarly to WFAs, the output of a WTA is then a linear function of the state of the root. Formally, 
\begin{definition}
    A weighted tree automaton~(WTA) $\mathcal{A}$ with $n$ states on $\trees_{\Sigma}$ is a tuple $\langle \vecs{\alpha} \in \Rbb^n, \ten{T}\in \Rbb^{n\times n \times n}, \{ \vec{v}_\sigma \in \Rbb^n  \}_{\sigma \in \Sigma} \rangle $. A WTA $\mathcal{A}$ computes a function $f_\mathcal{A} : \trees_\Sigma \to \Rbb$ defined by $f_\mathcal{A}(t) = \langle \vecs{\alpha}, \mu(t) \rangle$ where the mapping $\mu : \trees_\Sigma \to \Rbb^n$ is recursively defined by
    \begin{itemize}
        \item $\mu(\sigma) = \vec{v}_\sigma$ for all $\sigma\in\Sigma$,
        \item $\mu( (t_1,t_2) ) = \ten{T} \ttm{2} \mu(t_1) \ttm{3} \mu(t_2)$ for all $t_1,t_2\in\trees_\Sigma$.
    \end{itemize}
\end{definition}
The states of an $n$ state WTA are thus the $n$-dimensional vectors $\mu(\tau)$ for each subtree $\tau$ of the input tree. The computation of states of a WTA on an exemple tree is illustrated in Figure~\ref{fig:wta}~(left). 

WTAs are naturally related to weighted~(and probabilistic) context free grammars in the following ways. First, the set of derivation trees of a context free grammar is a regular tree language, that is a language that can be recognized by a tree automaton~\citep{magidor1970probabilistic,tata}. Furthermore, a weighted context free grammar~(WCFG)) maps any sequence to a real value by summing the weights of all valid derivation trees of the input sequence, where the weight of a tree is the value computed by a given WTA~(which defines the WCFG)~\citep{droste}. 

\subsection{Transformers}
\label{transformers}
The transformer architecture used in our construction is very similar to the encoder in the original transformer architecture \citep{transformers}.
First, we define the self-attention mechanism as
\begin{align*}
    f(\mathbf{X}) = \text{softmax}(\mathbf{X}\mathbf{W}_Q\mathbf{W}_K^\top \mathbf{X}^\top)\mathbf{X}\mathbf{W}_V,
\end{align*}
where $\mathbf{W}_Q, \mathbf{W}_K, \mathbf{W}_V \in \mathbb{R}^{d \times k}$, $d$ is the embedding dimension and $k$ is some chosen dimension, usually with $k < d$. Note that the softmax is taken row-wise.

One can also define the self attention layer using \textit{hard attention} by setting the largest value row-wise to 1 and all others to 0. This can be thought of as each row "selecting" a specific token to attend to.

By taking $h$ copies of this structure, concatenating the outputs of each head and applying a linear layer, we obtain a multi-head attention block, which we denote $f_{\text{attn}}$. 

We can now define the full transformer architecture.
Given a sequence of length $T$ with embedding dimension $d$, an \textit{$L$-layer transformer} is a sequence to sequence network $f_{\text{tf}}:\R^{T \times d} \to \R^{T \times d}$ where each layer is composed of a multi-head attention block followed by a feedforward block in the following manner
\begin{align*}
    f_{\text{ft}} = f_{\text{mlp}}^{(L)}\circ f_{\text{attn}}^{(L)} \circ
    f_{\text{mlp}}^{(L-1)}\circ f_{\text{attn}}^{(L-1)} \circ
    \hdots\circ f_{\text{mlp}}^{(1)}\circ f_{\text{attn}}^{(1)}.
\end{align*}
The feedforward block is simply a multilayer perceptron (MLP). The parameters of this MLP are not necessarily the same from one layer to another.

\subsection{Bilinear Layers}
We now introduce a special feed-forward layer which will be used in the construction for exact simulation of WFAs. 
\begin{definition}
Given two vectors $\mathbf{x}_1\in \R^{d_1}$ and $\mathbf{x}_2 \in \R^{d_2}$, a bilinear layer is a map from $\R^{d_1} \times \R^{d_2}$ to $\R^{d_3}$ such that
\begin{align*}
    \text{BilinearLayer}(\mathbf{x}_1,\mathbf{x}_2) = 
    \ten{T}\times_1\mathbf{x}_1\times_2\mathbf{x}_2 + \mathbf{b}
\end{align*}
where $\ten{T} \in \R^{d_1 \times d_2 \times d_3}$ and $\mathbf{b} \in \R^{d_3}$ are learnable parameters.
\end{definition}
It is easy to see that any bilinear map can be computed by such a layer given that all bilinear maps can be represented as a tensors~(similarly to how linear maps are represented by matrices). Note that bilinear layers have been introduced and used for practical applications previously, especially in the context of multi-modal and multiview learning~\citep{gao2016compact,li2017factorized,lin2015bilinear}.

\begin{figure*}[th]
\begin{center}
\includegraphics[width=0.95\textwidth]{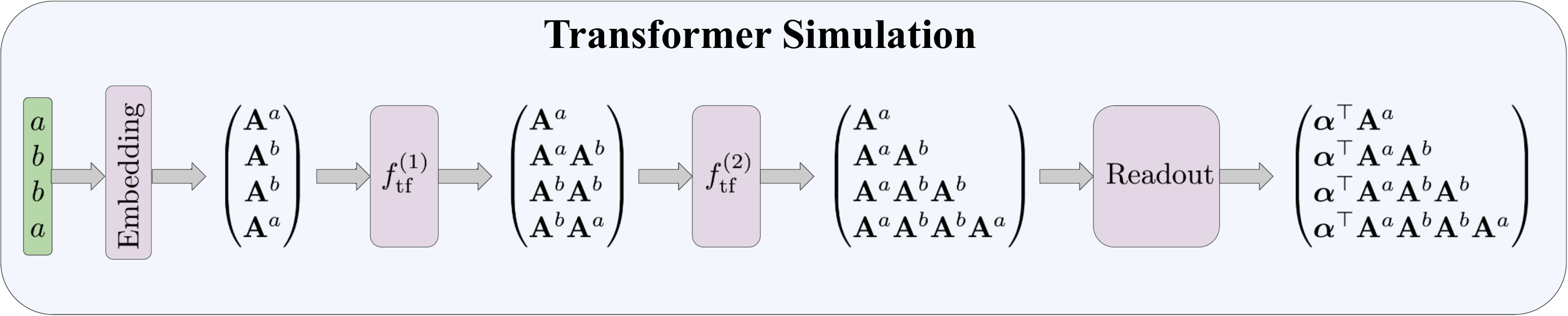}
\end{center}
\caption{Simulation of the WFA computation over the input $w=abba$ with a transformer.}\label{fig:wfa}
\end{figure*}

\section{SIMULATING WEIGHTED AUTOMATA OVER SEQUENCES}
In this section, we start by introducing the definition of simulation for WFAs, which is crucial to understanding the results in this section. We then state and briefly analyze our main theorems.

\subsection{Simulation Definition}
Intuitively, simulation can be thought of as reproducing the intermediary steps of computation for a given algorithm. For a WFA, these intermediary steps correspond to the state vectors throughout the computation over a given word.

\begin{definition}
Given a WFA $\mathcal{A}$ over some alphabet $\Sigma$, a function $f:\Sigma^T \to \R^{T \times n}$ \textit{exactly} simulates $\mathcal{A}$ at length $T$ if, for all $x \in \Sigma^T$ as input, we have $f(x) = \mathcal{A}(x)$, where $\mathcal{A}(x) = (\boldsymbol{\alpha}^\top, \boldsymbol{\alpha}^\top\mathbf{A}^{x_1}, \hdots, \boldsymbol{\alpha}^\top\mathbf{A}^{x_{1:T}})^\top$.
\end{definition}

Additionally, we define the notion of approximate simulation. Intuitively, given some error tolerance $\epsilon$, we can always find a function $f$ which can simulate a WFA up to precision $\epsilon$.

\begin{definition}
Given a WFA $\mathcal{A}$ over some alphabet $\Sigma$, a function $f:\Sigma^T \to \R^{T \times n}$ \textit{approximately} simulates $\mathcal{A}$ at length $T$ with precision $\epsilon > 0$ if for all $x \in \Sigma^T$, we have $\| f(x) - \mathcal{A}(x) \|_F < \epsilon$.
\end{definition}

Using a $T$ layer transformer, it is easy to simulate a WFA over a sequence of length $T$. We simply use the transformer as we would an unrolled RNN; performing each step of the computation at the corresponding layer. 

However, transformers are typically very shallow networks \citep{brown2022wide}, which defies intuition given how deep models tend to be more expressive than their shallow counterparts \citep{eldan2016power,cohen2016expressive}. This naturally leads to the following question: can transformers simulate WFAs using a number of layers that is less than linear (in the sequence length)? Following the work of \cite{liu2022transformers}, we define the notion of shortcuts.

\begin{definition}
    Let $\mathcal{A}$ be a WFA. If for every $T \geq 0$, there exists a transformer $f_T$ that simulates (exactly or approximately) $\mathcal{A}$ at length T with depth $L \leq o(T)$, then we say that there exists a shortcut solution to the problem of simulating $\mathcal{A}$.
\end{definition}
\begin{figure*}[th]
\begin{center}
\includegraphics[width=0.95\textwidth]{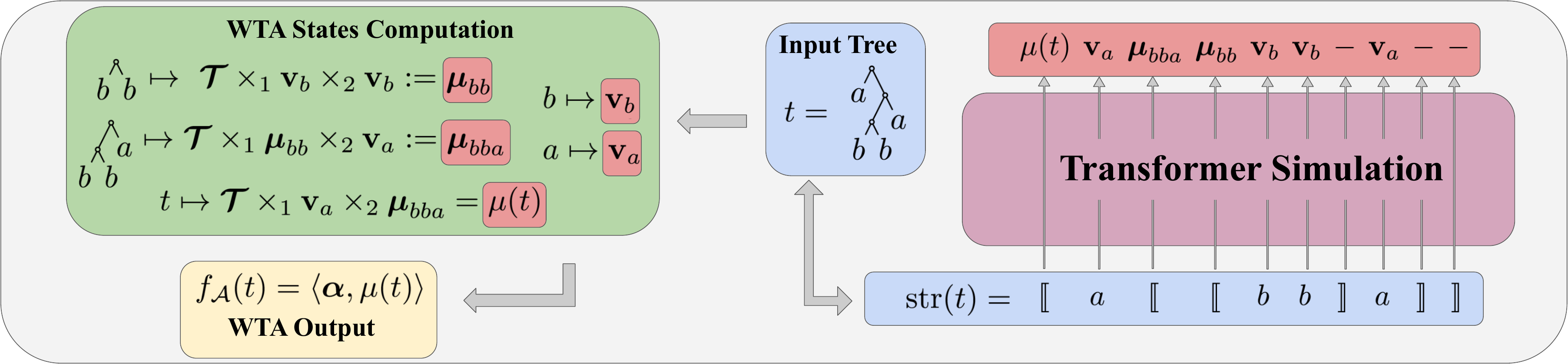}
\end{center}
\caption{Computation of a WTA on the input tree  $t=(a,((b,b),b))$ (left) and simulation of the WTA computation over $t$ with a transformer~(right).}\label{fig:wta}
\end{figure*}

\subsection{Main theorems}
In the following section, we state our main theorems and discuss their scope as well as their limitations. The proofs of these theorems can be found in Appendix \ref{appendix:B}. 

\begin{theorem}
Transformers using bilinear layers in place of an MLP and hard attention can \textit{exactly} simulate all WFAs with $n$ states at length $T$, with depth $\mathcal{O}(\log T)$, embedding dimension $\mathcal{O}(n^2)$, attention width $\mathcal{O}(n^2)$,  MLP width $\mathcal{O}(n^2)$ and $\mathcal O(1)$ attention heads.
\label{thmexact}
\end{theorem}

The proof of this theorem relies on hard attention as well as the use of bilinear layers. Since this does not correspond to the typical definition of transformer used in practice, we derive a second result. This consists in an approximate version of the previous theorem and relies on a more standard transformer implementation, using a softmax and a standard feedforward MLP.

\begin{theorem}
Transformers can \textit{approximately} simulate all WFAs with $n$ states at length $T$, up to arbitrary precision $\epsilon > 0$, with depth $\mathcal{O}(\log T)$, embedding dimension $\mathcal{O}(n^2)$, attention width $\mathcal{O}(n^2)$,  MLP width $\mathcal{O}(n^4)$  and $\mathcal O(1)$ attention heads.
\label{thmapprox}
\end{theorem}
Notice that the size of the construction does not depend on the approximation error $\epsilon$. This is one of the advantages of our approximate construction: we can achieve arbitrary precision without compromising the size of the model.

The proofs of both theorems rely on the \textit{prefix sum algorithm} ~\citep{blelloch1990prefix}, an algorithm that can compute all $T$ prefixes of a sequence in $\mathcal{O}(\log T)$ time (for more details we refer the reader to the appendix). Formally, this means that given a sequence $x_1, x_2, \hdots, x_T$ as input, the algorithm returns all partial sums $(x_1), (x_1 + x_2), \hdots, (x_1 + \hdots + x_T)$.

For both proofs in this section, we consider sequences of transition maps, and use composition as our "sum" operation. Using the definition of the transformer given in Section~\ref{transformers}, we present a construction which implements this algorithm. Figure \ref{fig:wfa} illustrates the key elements of this construction. Here, we take $f_\text{tf}^{(\ell)}$ to be the $\ell$th layer of an $L$ layer transformer such that $f_\text{tf}^{(\ell)} = f_{\text{mlp}}^{(\ell)}\circ f_{\text{attn}}^{(\ell)}$.

\section{SIMULATING WEIGHTED TREE AUTOMATA}\label{sec:WTA}

We now turn our focus to analyzing the capacity of transformers to efficiently simulate weighted tree automata. 
\subsection{Simulation definition}
As mentioned in Section~\ref{subsec:wta.prelim}, the states of a WTA of size $n$ are the $n$-dimensional vectors $\mu(\tau)$ for all subtrees $\tau$ of the input tree $t$. Intuitively, we will say that a transformer can simulate a given WTA if it can compute all subtree states  $\mu(\tau)$ when fed as input a string representation of $t$. We now proceed to formalize this notion of simulation.

Given a tree $t\in\trees_\Sigma$, we denote by $\str(t)$ its string representation, omitting commas. More formally, $\str(t) \in \Sigma^*\cup\{\opensymbol,\closesymbol\}$ is recursively defined by
 \begin{itemize}
        \item $\str(\sigma) = \sigma$ for all $\sigma\in\Sigma$,
        \item $\str( (t_1,t_2) ) = \opensymbol\ \str(t_1)\ \str(t_2)\  \closesymbol$ for all $t_1,t_2\in\trees_\Sigma$.
\end{itemize}
It is worth mentioning that the mapping $t\mapsto \str(t)$ 
is injective, thus any tree $t$ can be recovered from its string representation~(but not all strings in $(\Sigma\cup\{\opensymbol,\closesymbol\})^*$ are valid representations of trees).
 
One can then observe that each opening parenthesis in $\str(t)$ can naturally be mapped to a subtree of $t$. 
Formally, given a tree $t$, let $\mathcal{I}_t = \{ i = 1,\cdots, |\str(t)|  \mid  \str(t)_i \not = \closesymbol \}  $ be the set of positions of $\str(t)$ corresponding to opening parenthesis and leaf symbols. Then, the set of subtrees of $t$ can be mapped (one-to-one) to indices in $\mathcal{I}_t$:
\begin{definition}\label{def:subtree.str.rep}
Given a tree $t \in \trees_\Sigma$, we let $\tau : \mathcal{I}_t \to \trees_\Sigma$ be the mapping defined by 
$$\tau(i) = \str^{-1}(x_ix_{i+1}\cdots x_j)$$ where $j$ is the unique index such that the string $x_ix_{i+1}\cdots x_j$ is a valid string representation of a tree~(i.e. $s = x_ix_{i+1}\cdots x_j$ is the only substring starting in position $i$ such that there exists a tree $\tau\in\trees_\Sigma$ for which $\str(\tau)=s$). 
\end{definition}

Intuitively, we will say that a transformer can simulate a WTA if, given some
input $\str(t)$, the output of the transformer for each position $i\in\mathcal{I}_t$ is equal to the corresponding state $\mu(\tau(i))$ of the WTA after parsing the subtree $\tau_i$. We can now formally define the notion of WTA simulation.

\begin{definition}
    Given a WTA $\mathcal{A}=\langle \vecs{\alpha}, \ten{T}, \{ \vec{v}_\sigma \}_{\sigma \in \Sigma} \rangle$ with $n$ states on $\trees_{\Sigma}$, we say that a function $f: (\Sigma\cup \{\opensymbol,\closesymbol \})^T \to (\Rbb^n)^T$ simulates $\mathcal{A}$ at length $T$ if for all trees $t\in\trees_\Sigma$ such that $|\str(t)|\leq T$, $f(\str(t))_i = \mu(\tau_i)$ for all $i\in\mathcal{I}_t$ (where the subtrees $\tau_i$ are defined in Def.~\ref{def:subtree.str.rep}).
    
    Furthermore, we say that a family of functions $\mathcal{F}$ simulates WTAs with $n$ states at length T if for any WTA $\mathcal{A}$ with $n$ states there exists a function $f\in\mathcal{F}$ that simulates $\mathcal{A}$ at length $T$.
\end{definition}
The overall notion of WTA simulation is illustrated in Figure~\ref{fig:wta}. Note that this definition could be modified to encode a tree with the closing brackets instead of the opening ones. In this case, our results would still hold using attention layers with causal masking.

\subsection{Results}

We are now ready to state our main results for weighted tree automata:
\begin{theorem}
\label{thm:approx.wta}
Transformers can \textit{approximately} simulate all WTAs $\mathcal{A}$ with $n$ states at length $T$, up to arbitrary precision $\epsilon > 0$, with embedding dimension $\mathcal{O}(n)$, attention width $\mathcal{O}(n)$,  MLP width $\mathcal{O}(n^3)$  and $\mathcal O(1)$ attention heads\footnote{The big O notation does not hide any large constant here: the depth is exactly $1+\mathrm{depth(t)}$, the embedding dimension and the attention width are $n+4+p$ (where $p$ is the size of the positional embedding) and the MLP width is $\frac{1}{2}(2n+1)(2n+2)$. }. Simulation over arbitrary trees can be done with depth $\mathcal{O}(T)$ and simulation over balanced trees~(trees whose depth is of order $\log(T)$) with depth $\mathcal{O}(\log(T))$.
\end{theorem}
The construction used in the the proof~(which can be found in Appendix \ref{Appendix:C}) revolves around two main ideas: (i) leveraging the attention mechanism to have each node attend to the positions corresponding to its left and right subtrees and (ii) using the feedforward layers to approximate the bilinear mapping $(\mu(t_1),\mu(t_2))\mapsto \mu((t_1,t_2)) = \ten T \times_1 \mu(t_1)\times_2 \mu(t_2)$. In the intial embedding, each leaf position is initialized to the corresponding leaf state $\vec v_\sigma$. At each layer of computation, each position computes the output of the multilinear map applied to the embeddings of its respective left and right subtrees.  Thus, after one layer the state of all the subtrees of depth up to 2 have been computed~(leafs and subtrees of the form $(\sigma_1,\sigma_2))$. Similarly, after the $\ell$th layer, the transformer will have simulated all the states corresponding to subtrees of depth up to $\ell +1$. Thus after as many layers as the depth of the input tree, all states have been computed.

We will conclude with a few interesting observations. First, note that for all balanced tree, all states will have been computed after $\mathcal O(\log T) $ layers of computation. Thus, if we only consider well-balanced trees, i.e. trees whose depth is in $\mathcal O(\log T) $, then our construction constitutes a shortcut. 
However, in the case of the most extremely unbalanced tree~(which is  a comb, i.e. a tree of the form $(\sigma_1,(\sigma_2,(\sigma_3,(\cdots))))$), it will take $\mathcal O(T)$ layers to compute all states. At the same time, in this most extreme case, a tree is nothing else than a string/sequence, and the computation of the WTA on this tree can be as well be carried out by a WFA, which could be simulated by a transformer with $\mathcal O(\log T)$ layers from our previous results for WFA. This raises the question of whether there exist a construction interpolating between the WFA and WTA constructions proposed in this paper which could simulate WTA with $\mathcal O(\log T)$ layers for \emph{arbitrary} trees. Still, this result shows that  transformers can indeed learn shortcuts to WTAs when restricting the set of inputs to balanced trees.

Second, since WTAs subsume classical finite state tree automata, an important corollary of our result is that transformers can also simulate~(non-weighted) tree automata. Classical tree automata can be defined as WTA with weights in the Boolean semi-ring. Intuitively, Boolean computations can, in some sense, be simulated with real weights by interpreting any non-zero value as true and any zero value as false, thus WTA can simulate classical tree automata. Since we just showed that WTA can in turn be simulated by transformers, we have the following corollary~(whose proof can be found in appendix).
\begin{corollary}
\label{cor:approx.treeautomata}
Transformers can \textit{approximately} simulate all tree automata $\mathcal{A}$ with $n$ states at length $T$, up to arbitrary precision $\epsilon > 0$, with the same hyper-parameters and depth as for WTA given in Theorem~\ref{thm:approx.wta}.
\end{corollary}

\begin{table*}[h]
    \centering
    \begin{tabular}{ccccccccccll} 
        \toprule
         Pautomac nb&  4&  12&  14&  20&  30&  31&  33&  38& 39 & 45\\ 
     \midrule
         num states&  12&  12&  15&  11&  9&  12&  13&  14&  6& 14\\ 
         alphabet size&  4&  13&  12&  18&  10&  5&  15&  10&  14& 19\\ 
         type&  PFA&  PFA&  HMM&  HMM&  PFA&  PFA&  HMM&  HMM&  PFA& HMM\\ 
         symbol sparsity & 0.4375 &  0.3526 &  0.4944 &  0.3939&  0.6555 & 0.3833&   0.5949& 0.7857 &  0.4167 &  0.8008\\ 
 nb layers for $\epsilon$& 8 & 6 & 2 & 6 & - & 8 & - & 2 & - & -\\ 
     \bottomrule
    \end{tabular}
    \caption{Minimum number of layers to reach error < $\epsilon=10^{-3}$}
    \label{tab:pautomac-table}
\end{table*}

\section{EXPERIMENTS}
\label{experiments}
In this section, we investigate if logarithmic shortcut solutions can be found using gradient descent based learning.
We train transformer models on sequence to sequence simulation tasks, where, for a given input sequence, the transformer must produce as output the corresponding sequence of states. We then study how varying certain parameters impacts model performance and compare this with results predicted by theory. We find that transformer models are indeed capable of approximately simulating WFAs and that the number of layers of the model has an important effect on performance, as predicted by theory.

\subsection{Can logarithmic solutions be found?}
We first investigate if, under ideal supervision, such solutions can even be found in the first place. We evaluate models on target WFAs taken from the Pautomac dataset \citep{pautomac}. We use the target automata to generate sequences of states with length $T=64$.

In Table \ref{tab:pautomac-table}, we report the minimum number of layers necessary for a transformer to reach a test (mean squared) error below $\epsilon=10^{-3}$. We report our findings for $L \in \{ 2, 4, 6, 8, 10\}$. If no model achieving $\epsilon$ was found, we indicate this using a dash. This table also includes all practical information about the target automaton such as its number of states and the size of its alphabet.

We see that for more than half of the considered automata, we are indeed able to find a solution which satisfies our error criterion. However, the minimum values attained are not always consistent with the logarithmic threshold for shortcuts we propose in theory~(since $T=64$, we would expect $L=6$ to be the threshold). Furthermore, for certain automata, the maximum number of layers considered does not seem to be sufficient to achieve the target error. This behavior is not consistent with the measures of complexity of the problem given in the table~(\textit{i.e.}, number of states, alphabet size, symbol sparsity). 
In future work, it would be interesting to see if a more thorough hyperparameter search could lead to finding shortcut solutions for these automata as well.

\subsection{Do solutions scale as theory suggests?}
Next, we investigate how such solutions scale as key parameters in our construction increase. We study how the number of layers as well as the size of the embedding dimension influence the performance of transformer models trained to simulate synthetic WFAs.

For the experiments concerning the number of layers, the data is generated using a WFA with 2 states which counts the number of 0s in a binary string with $\Sigma=\{ 0,1\}$. Figure \ref{fig:avg_mse_layers} shows how the mean squared error (MSE) varies as we increase the number of layers. We consider $T \in \{16, 32, 64 \}$, and each curve represents a sequence length for which we run the experiment. The dotted vertical lines represent the theoretical value for shortcuts to be found, \textit{i.e.} $\log_2(T)$.

For all sequence lengths, the error curves display a pronounced elbow. We see that, at first, increasing the number of layers has a notable effect on decreasing the MSE. However, after a certain threshold, the MSE seems to stabilize and adding more layers has negligible effect. It is also interesting to note how close this stabilization point is to the number of layers necessary for shortcut solutions in practice.

For the experiments on the embedding dimension, we generate data using a WFA which counts $k$ distinct symbols in sequences over some alphabet $\Sigma$. One can define such a WFA using $k+1$ states, where the first $k$ components of the state correspond to the current counts for each of the $k$ symbols, and the last component is constant and equal to one. As an example, consider $\Sigma = \{ a,b,c\}$, we could define a 2-counting automata which counts $a$ and $b$. For the word $w = aaabb$ such an automaton would return the state
$\begin{pmatrix}
    3,2,1\\
\end{pmatrix}$,
where the two first dimensions count $a$ and $b$ respectively. More details about the considered automata can be found in the appendix.

For the purpose of this experiment, we consider $\Sigma = \{0,1,2,...,9 \}$ and $k \in \{2,4,6,8 \}$, where we count the $k$ first characters in the alphabet (in numerical order). For simplicity, we fix both the embedding dimension and the hidden size of the model to the same value.
 \begin{figure}[t]
\begin{center}
    \includegraphics[scale=1]{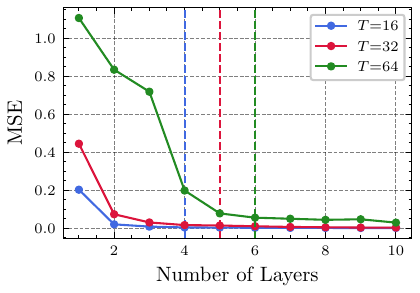}	
\end{center}
\caption{Average MSE vs. number of layers: For all considered sequence lengths, adding layers has an notable effect on the MSE at first, however past a certain point, the improvement is negligible. This stabilization is consistent with our theoretical results (shown as dotted lines).}
\label{fig:avg_mse_layers}
\end{figure}

\begin{figure}[t]
\begin{center}
    \includegraphics[scale=1]{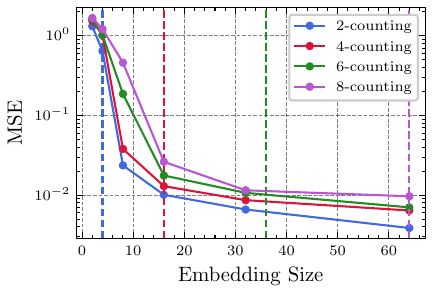}	
\end{center}
\caption{Average MSE (log scale) vs. embedding size: Increasing the embedding size also has a notable effect on the MSE. However the stabilization of the curves does not agree with as closely with our theoretical results (shown as dotted lines).}
\label{fig:avg_mse_embed}
\end{figure}

The results are presented in Figure~\ref{fig:avg_mse_embed}, where we also notice that the curves show a pronounced elbow shape. Interestingly, we note that the stabilisation of the error does not seem to follow the value $n^2$ predicted by the theory as closely as for the number of layers. This may be due to the training procedure considered. For instance, training on a bigger dataset, for more epochs or using more varied hyperparameters, the results may match more closely the trend predicted by the theory.

\section{RELATED WORKS}
\paragraph{Formal languages and Neural Networks}
There has been extensive study of the relationship between formal languages and neural networks. Many studies investigate the empirical performances of sequential models on formal language recognition tasks. \cite{deletang2022neural, bhattamishra2020ability} show where transformers and RNN variants lie on the Chomsky hierarchy by training these models on simple language recognition tasks. \cite{ebrahimi2020can} study the empirical performance of transformers on Dyck languages, which describe balanced strings of brackets. \cite{merrill2020formal} give an alternative hierarchy using space complexity and rational recurrence. There are also many theoretical results concerning connections between formal languages and neural models. \cite{chiang2022overcoming, daniely2020learning} show constructions of feedforward neural networks and transformers for PARITY (a language of binary strings such that the number of 1s is even) and FIRST (a language of binary strings starting with a 1). \cite{yao2021self}, on the other hand, give constructions of transformers for Dyck languages.

\paragraph{Neural networks and models of computation}
Work has also been done investigating theoretical connections between classical computer science models and neural networks. The most widely-known result is certainly \cite{chung2021turing}, showing that RNNs are Turing-complete. However, this work assumes unbounded computation time and infinite precision and thus such a result is not very applicable in more practical settings. Several work focus on extracting DFAs \citep{WeissDFA18,Giles,OmlinGiles,merrill2022,muskardin} and WFAs \citep{WeissWFA19,Takamasa,zhangaaai} from recurrent neural networks. A similar, more general approach can be found in a line of work focusing on extraction from black-box models on sequential data \citep{Ayache2018,eyraud2020,AAK-RNN} and on knowledge distillation from RNNs and transformers \citep{taysir}. Interestingly, it is possible to formally show that WFA are equivalent to 2-RNNs, a family of RNNs which uses bilinear layers \citep{li2022connecting}. Moreover, connections have been shown between transformers and Boolean circuits. \cite{merrill2022saturated, hao2022formal} and \cite{chiang2023tighter} show that transformers can implement Boolean circuits and show which circuit complexity classes transformers can recognize. On another line of thought, \cite{weiss_transf} show how transformers can realize declarative programs.

Closest to our work is \cite{liu2022transformers}, which first introduced the notion of transformers simulating DFAs. This paper initially prompted us to investigate what other families of automata could be simulated by transformers. Moreover, many of the technical ideas used in Theorem \ref{thmexact} and Theorem \ref{thmapprox} are inspired by the proof they present for the logarithmic case. 

Lastly, \cite{zhao2023transformers} is also very relevant to our work. Their work shows that transformers can simulate probabilistic context free grammars (PCFGs). More precisely, they show that transformers can implement the inside-outside algorithm~\citep{baker1979trainable}, which is a dynamic programming algorithm used to compute the probability of a given sequence under a PCFG. I.e., to compute the sum of the values returned by a given WTA~(which defines the PCFG) on all possible derivation trees of the input sequence: a task orthogonal~(and more difficult) than the one of simulating a WTA we consider here. Their construction~(understandably) requires more parameters than ours~($\mathcal O(T)$ layers with $\mathcal O(n)$ attention heads each and embedding size of $\mathcal O(nT)$).

\section{CONCLUSION}
In this paper, we theoretically demonstrate that transformers can simulate both WFAs and WTAs. For WFAs, simulation can be achieved using a number of layers logarithmic in the sequence length, whereas for WTAs, the number of layers must be equivalent to the depth of the tree given as input. 
Our results extend the ones of \cite{liu2022transformers} showing that transformers can learn shortcuts to models significantly more complex than deterministic finite automata.
We verified on simple synthetic experiments that  such shortcut solutions to WFAs can indeed be found using gradient based training methods. We hope that our results may shed some light on the success of transformers for sequential reasoning tasks, and give practical considerations in terms of depth and width of such models for given tasks. 

There are many theoretical and empirical directions in which our work could be extended. A first question is whether such shortcut solutions can be found \textit{in the wild}. Do transformer models trained on downstream tasks implement exactly or approximately the algorithmic reasoning capabilities of WFAs or WTAs? It may be interesting to analyze to what extent transformers natively implement the algorithmic reasoning used in our constructions.
Concerning theoretical results, we only provide upper bounds to the simulation capacities of the considered models. One could extend the results presented in this paper by deriving lower bounds for either WFA or WTA. We posit that there should exist languages for which these bounds are tight.
Another interesting question would be to analyze how such solutions scale with sample complexity and optimization schemes. Our results are about the existence of shortcut solutions, but the question of learnability for such solutions remains open. Empirically or theoretically, it would be interesting to show how the quantity of data, optimization procedure, or various aspects of the target structure can affect the quality of found shortcuts. In the same line of thought, an analysis of the training dynamics may be interesting. 

\section*{Acknowledgements}
This research is supported by the Canadian Institute
for Advanced Research (CIFAR AI chair program) and by the Natural Sciences and Engineering Research Council of Canada (NSERC).

\bibliography{bibliography.bib}
\bibliographystyle{unsrtnat}


\onecolumn
\newpage
\section*{Checklist}

 \begin{enumerate}

 \item For all models and algorithms presented, check if you include:
 \begin{enumerate}
   \item A clear description of the mathematical setting, assumptions, algorithm, and/or model. \textbf{Yes}
   \item An analysis of the properties and complexity (time, space, sample size) of any algorithm. \textbf{Yes}
   \item (Optional) Anonymized source code, with specification of all dependencies, including external libraries. \textbf{Not Applicable}
 \end{enumerate}

 \item For any theoretical claim, check if you include:
 \begin{enumerate}
   \item Statements of the full set of assumptions of all theoretical results. \textbf{Yes}
   \item Complete proofs of all theoretical results. \textbf{Yes}
   \item Clear explanations of any assumptions. \textbf{Yes}     
 \end{enumerate}

 \item For all figures and tables that present empirical results, check if you include:
 \begin{enumerate}
   \item The code, data, and instructions needed to reproduce the main experimental results (either in the supplemental material or as a URL). \textbf{No}, but we intend on making the code repository public once the AISTATS decision will be made.
   \item All the training details (e.g., data splits, hyperparameters, how they were chosen). \textbf{Yes}
    \item A clear definition of the specific measure or statistics and error bars (e.g., with respect to the random seed after running experiments multiple times). \textbf{Yes}
     \item A description of the computing infrastructure used. (e.g., type of GPUs, internal cluster, or cloud provider). \textbf{Yes}
 \end{enumerate}

 \item If you are using existing assets (e.g., code, data, models) or curating/releasing new assets, check if you include:
 \begin{enumerate}
   \item Citations of the creator if your work uses existing assets. \textbf{No}
   \item The license information of the assets, if applicable. \textbf{Not Applicable}
   \item New assets either in the supplemental material or as a URL, if applicable. \textbf{Not Applicable}
   \item Information about consent from data providers/curators. \textbf{Not Applicable}
   \item Discussion of sensible content if applicable, e.g., personally identifiable information or offensive content. \textbf{Not Applicable}
 \end{enumerate}

 \item If you used crowdsourcing or conducted research with human subjects, check if you include:
 \begin{enumerate}
   \item The full text of instructions given to participants and screenshots. \textbf{Not Applicable}
   \item Descriptions of potential participant risks, with links to Institutional Review Board (IRB) approvals if applicable. \textbf{Not Applicable}
   \item The estimated hourly wage paid to participants and the total amount spent on participant compensation. \textbf{Not Applicable}
 \end{enumerate}

 \end{enumerate}
 \newpage

\aistatstitle{Simulating Weighted Automata over Sequences and Trees with Transformers \\ 
\vspace{0.5cm}

$-$ Supplementary Material $-$}
\appendix
\section{Background and Notation}\label{appendix:A}
\subsection{The Recursive Parallel Scan Algorithm}
In this section, we give an in-depth explanatation of the recursive parallel scan algorithm, which is crucial to the construction in the two first theorems in this paper.
The recursive parallel scan algorithm or prefix sum algorithm is an algorithm which calculates the running sum of a sequence \cite{blelloch1990prefix} Given a sequence of numbers $\{x_0, x_1, x_2, \hdots x_n\}$, the method outputs the running sum 
\begin{align*}
    y_0 &= x_0\\
    y_1 &= x_0 \oplus x_1\\
    \vdots\\
    y_n &= x_0 \oplus x_1 \oplus \hdots \oplus x_n.\\
\end{align*}
Where $\oplus$ represents any associative binary operator such as addition, string concatenation or matrix multiplication, for example. In order to compute this running sum, the algorithm uses a divide-and-conquer scheme. The pseudocode for the algorithm is given below.
\begin{algorithm}
\caption{Prefix Sum (Scan) Algorithm}
\KwData{Input sequence $\{x_0, x_1, x_2, \hdots x_n\}$}
\For{$i \leftarrow 0$ \KwTo $\text{floor}(\log_2 n$)}{
    \For{$j \leftarrow 0$ \KwTo $n - 1$}{
        \eIf{$j < 2^i$}{
            $x_j^{i + 1} \gets x_j^i$
        }
        {
            $x_j^{i + 1} \gets x_j^i + \gets x_{j - 2^i}^i$
        }
        
    }
}
\end{algorithm}

Here, we take $x_j^i$ to be the $j$th element in the sequence at timestep $i$.

The key idea of this algorithm is to recursively compute smaller partial sums at each step and combine them in the step after. The first for loop iterates through the powers of two which determine how far apart each precomputed sum is to the next (indexed by $i$). The second for loop recursively combines all partial sums that are powers of $2^i$ apart. By doing this a total of $\log n$ times, we are able to compute the running sum. Figure \ref{fig:prefix_sum} gives an example of this procedure for a sequence of length 8.

Considering there are $\mathcal{O}(n)$ sums to compute at every step, this leads to a total runtime of $\mathcal{O}(n\log n)$. In terms of space complexity, such an algorithm can be executed in place, by storing the results of each intermediate computation in the input array, thus leading to a space complexity of $\mathcal{O}(1)$.
\begin{figure}[H]\label{fig:prefix_sum}
\begin{center}
    \includegraphics[scale=0.5]{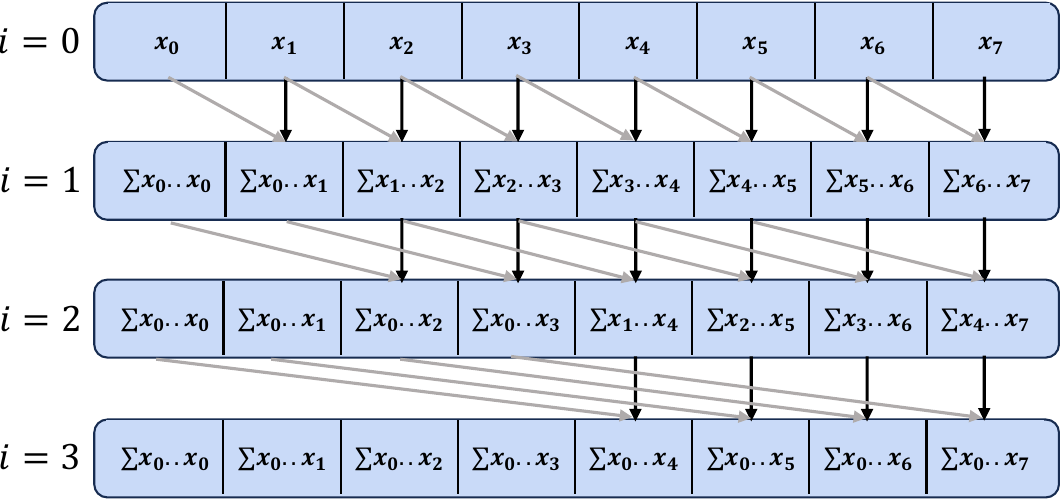}
\end{center}
\caption{Illustration of the prefix sum algorithm}
\end{figure}

\subsection{Weighted Finite Automata}
In this section, we give a more thorough treatment of WFAs and detail the families/instances of automata considered in the experiments section. We start by recalling the definition of a WFA
    \begin{definition}
    A \emph{weighted finite automaton} (WFA) of $n$ states over $\Sigma$ is a tuple $\mathcal{A} = \langle \boldsymbol{\alpha} , \{\mat{A}^\sigma\}_{\sigma \in \Sigma},  \boldsymbol{\beta} \rangle$, where $\boldsymbol{\alpha},$ $\boldsymbol{\beta} \in \R^n$ are the initial and final weight vectors, respectively, and $\mat{A}^\sigma \in \R^{n \times n}$ is the matrix containing the transition weights associated with each symbol $\sigma \in \Sigma$. Every WFA $\mathcal{A}$ with real weights realizes a function $f_A : \Sigma^* \to \R$, \emph{i.e.} given a string  $x = x_1 \cdots x_t \in \Sigma^*$, it returns $f_\mathcal{A}(x) = \boldsymbol{\alpha} ^\top \mat{A}^{x_1} \cdots \mat{A}^{x_t} \boldsymbol{\beta} = \boldsymbol{\alpha} ^\top \mat{A}^x \boldsymbol{\beta}$. 
    \end{definition}

\paragraph{Hidden Markov models}
Hidden Markov models (HMMs) are statistical models that compute the probability of seeing a sequence of observable variables subject to some "hidden" variable which follows a Markov process. It is possible to show that HMMs are a special case of WFAs. Recall the definition of a HMM
\begin{definition}
    Given a set of states $S = \{1, \hdots, n\}$, and a set of observations $\Sigma = \{1, \hdots, p \}$, a HMM is given by
    \begin{itemize}
        \item Transition probabilities $\mat{T} \in \R^{n \times n}$, where 
        $\mat{T}_{ij} = \mathbb{P}(h_{t+1}=j \mid h_t = i)$;
        \item Observation probabilities $\mat{O} \in \R^{p \times n}$, where $\mat{O}_{ij} = \mathbb{P}(o_t = i \mid h_t = j)$;
        \item An initial distribution $\boldsymbol{\pi} \in \R^n$, where $\boldsymbol{\pi}_i = \mathbb{P}(h_1 = i)$.
    \end{itemize}
\end{definition}
A HMM defines a probability distribution over all possible sequences of observations. The probability of a given sequence $x_1x_2x_3\hdots x_k$ is given by
\begin{align*}
    \mathbb{P}(x_1x_2x_3\hdots x_k) = \sum_{i_1}\boldsymbol{\pi}_i\mat{O}_{x_1,i_1}
    \sum_{i_2}\mat{T}_{i_1,i_2}\mat{O}_{x_2,i_2}\hdots
    \sum_{i_k}\mat{T}_{i_{k-1},ik}\mat{O}_{x_k,i_k}.
\end{align*}
Here, notice the similarities between this computation and that of a WFA. This lets us rewrite the HMM computation in terms of the definition of a WFA:
\begin{align*}
    \boldsymbol{\alpha} &= \boldsymbol{\pi}\\
    \mat{A}^x &= \text{diag}(\mat{O}_{x,1},\hdots,\mat{O}_{x,n})\mat{T},\; \forall x \in \Sigma\\
    \boldsymbol{\beta} &= \boldsymbol{1} \in \R^{n}.
\end{align*}
Where $\boldsymbol{1}$ is taken to be the vector full of ones. Thus, we have that
\begin{align*}
    \Prob(x_1x_2x_3\hdots x_k) = \boldsymbol{\alpha}^\top \mat{A}^{x_1}\mat{A}^{x_2}\mat{A}^{x_3}\hdots \mat{A}^{x_k}\boldsymbol{\beta}.
\end{align*}
\paragraph{Probabilistic Finite Automata}
Probabilistic finite automata (PFA) \cite{vidal2005probabilistic} are another subcategory of WFAs which compute probabilities. Here, we do not assume the matrices are row-stochastic as is the case for HMMs.
\begin{definition}
    A PFA is a tuple $\mathcal{A}=\langle Q_\mathcal{A}, \Sigma,\delta_\mathcal{A}, I_\mathcal{A}, F_\mathcal{A}, P_\mathcal{A} \rangle$, where:
    \begin{itemize}
        \item $Q_\mathcal{A}$ is a finite set of states;
        \item $\Sigma$ is the alphabet;
        \item $\delta_\mathcal{A} \subseteq Q \times \Sigma \times Q$;
        \item $I_\mathcal{A}: Q_\mathcal{A} \to \R^+$;
        \item $P_\mathcal{A}: \delta_\mathcal{A} \to \R^+$;
        \item $F_\mathcal{A}: Q_\mathcal{A} \to \R^+$.
    \end{itemize}
    $I_\mathcal{A}, P_\mathcal{A}$ and $F_\mathcal{A}$ are functions such that:
    \begin{align*}
        \sum_{q\in Q} I_\mathcal{A}(q) = 1
    \end{align*}
    and
    \begin{align*}
        \forall q \in Q_\mathcal{A},\; F_\mathcal{A}(q) + \sum_{\sigma \in \Sigma, q'\in Q_\mathcal{A}} P_\mathcal{A}(q,\sigma,q') = 1.
    \end{align*}
\end{definition}
Assuming we have $Q_\mathcal{A} = \{0, \hdots, |Q|-1 \}$, we can relate PFAs to WFAs in the following way:
\begin{align*}
    \boldsymbol{\alpha}_i &= I_\mathcal{A}(i),\; \forall i \in Q_\mathcal{A}\\
    \mat{A}^\sigma_{ij} &= P_\mathcal{A}(i,\sigma,j)\,\; \forall i,j \in Q_\mathcal{A}\times Q_\mathcal{A}\\
    \boldsymbol{\beta}_i &= F_\mathcal{A}(i),\; \forall i \in Q_\mathcal{A}.
\end{align*}

\paragraph{Counting WFA}
The WFA which counts the number of 0s considered in the Section \ref{experiments} is defined with $n=2$ and $\Sigma = \{0,1 \}$. The parameters of this automaton are
\begin{align*}
    \boldsymbol{\alpha} = 
    \begin{pmatrix}
        0 & 1
    \end{pmatrix}^\top,\;
    \mat{A}^0 = 
    \begin{pmatrix}
        1 & 0 \\
        1 & 1
    \end{pmatrix},\;
    \mat{A}^1 = 
    \begin{pmatrix}
        1 & 0 \\
        0 & 1
    \end{pmatrix},\;
    \boldsymbol{\beta}=
    \begin{pmatrix}
        1 & 0
    \end{pmatrix}^\top.
\end{align*}
Observe that the matrix $\mat{A}^0$ has the following property
\begin{align*}
    \left(\mat{A}^0\right)^n =
    \begin{pmatrix}
        1 & 0 \\
        n & 1
    \end{pmatrix},\;
\end{align*}
The matrix $\mat{A}^1$ on the other hand is the identity and leaves the count alone. Thus for some word $w \in \Sigma^*$, the state at the end of the computation would be
\begin{align*}
    \boldsymbol{\alpha^\top}\mat{A}^w = 
    \begin{pmatrix}
        |w|_0 & 1
    \end{pmatrix}.
\end{align*}

\paragraph{$k$-Counting WFA}
The $k$-counting WFA is a generalization of the automata presented in the previous paragraph. For the purpose of this work, we assume that we always count the $k$ first characters of $\Sigma$ and that $\Sigma = \{ 0, 1, \hdots N-1\}$, where $N = |\Sigma|$. Such an automaton needs $n=k+1$ states to execute such a computation. The parameters of this automaton are:
\begin{align*}
    \boldsymbol{\alpha} &=
    \begin{pmatrix}
        1 & 0 & \hdots & 0
    \end{pmatrix}
    \in \R^{k+1}\\
    \mat{A}^i &= \mat{I}_{k+1} + \vec{e}_{k+1}\vec{e}_i^\top,\;\forall i < k\\ 
        \mat{A}^i &= \mat{I}_{k+1},\;\forall i \geq k,
\end{align*}
with $i \in \Sigma$. The value of $\boldsymbol{\beta}$ is not important here as we are solely interested in the state. One could choose a terminal vector with 1s at specific positions to count the sum of certain symbols or even use 1s and -1s to verify if two symbols appear the same number of times.

\section{Simulating Weighted Finite Automata}\label{appendix:B}
\subsection{Proof of Theorem \ref{thmexact}}
First, we recall our first theorem. 
\begin{reptheorem}{thmexact}
Transformers using bilinear layers and hard attention can \textit{exactly} simulate all WFAs $\mathcal{A}$  at length $T$, with depth $\log T$, embedding dimension $2n^2 + 2$, attention width $2n^2 + 2$ and  MLP width $2n^2$, where $n$ is the number of states of $\mathcal{A}$.
\end{reptheorem}

Before diving into the details of our construction, we provide a high-level intuition of the proof. Similarly to \citep{liu2022transformers}, the key idea of this proof is using the recursive parallel scan algorithm to compute the composition of all transition maps efficiently. To implement this algorithm using a transformer, we store two copies of each symbol's transition map in the embeddings and use the attention mechanism to "shift" one of them by a power of two. 

Then we can use the MLP to calculate the matrix product between both transition maps to obtain the next set of values in the trajectory. By iterating this process for each of the $\log T$ layers, we are able to obtain every ordered combination of transition maps as it would be the case with the recursive parallel scan algorithm. 

\begin{proof}
We prove our result by construction. After recalling the main assumptions, we define each element in the construction. 

\paragraph{Important assumptions} We list our main assumptions.
\begin{itemize}
    \item For simplicity, we will only consider cases where the sequence length $T$ is a power of 2. Using padding, we can extend the construction to arbitrary $T$.
    \item We pad the input sequence with an extra $T$ tokens whose embedding are vectorized identities. This extra space will be used as a buffer to store the "shifted" version of the transition maps. The positions are indexed as $-T + 1, \hdots, 0, 1, \hdots, T$. This could equally be achieved using a more complicated MLP.
    \item Similarly to \citep{liu2022transformers}, our construction does not use any residual connections.
    \item This construction uses hard (or saturated) attention in the self attention block as well as a unique bilinear layer as the MLP block.
    \item We assume access to positional encodings at each layer. This could easily be implemented using either a third attention head (only two are used here) or using residual connections.
\end{itemize}
Please note that these assumptions are only for ease of exposition. It would be possible, but more complicated, to give a proof without them.
\paragraph{Embeddings} For the embeddings, we will use an embedding dimension of $d = 2n^2 + 2$, where $n$ is the number of states in the WFA. For a given symbol $\sigma_t, t \in [T]$, we define the embedding vector as:  
\begin{align*}
    \mathbf{x}_t &= 
    \begin{pmatrix}
        \text{vec}(\mathbf{A}^{\sigma_t}) & \text{vec}(\mathbf{A}^{\sigma_t}) & P_1(t) & P_2(t)
    \end{pmatrix}
    =
    \begin{pmatrix}
        \mathbf{x}_L & \mathbf{x}_R & \mathbf{x}_{P_1} & \mathbf{x}_{P_2}
    \end{pmatrix},
\end{align*}
with positional embeddings $P_1(t) = \cos \frac{\pi t}{T}$ and $P_2(t) = \sin \frac{\pi t}{T}$. For $t\leq 0$ we have $\mathbf{x}_t = 
    \begin{pmatrix}
        \text{vec}(\mathbf{I})&\text{vec}(\mathbf{I}) & P_1(t) & P_2(t)
    \end{pmatrix}$.
We refer to the dimensions associated to the first vectorized transition map as the left dimensions (indexed $\mathbf{x}_L$) and similarly the dimensions associated to the second vectorized transition map as the right dimensions (indexed $\mathbf{x}_R$). We also use indices $\mathbf{x}_{P_1}$ and $\mathbf{x}_{P_2}$ to refer to the positions associated with each positional encoding.

Note that the positional embeddings are defined for all values of $t \in \{ -T + 1, \hdots, 0, 1, \hdots, T \}$, meaning that they are also defined in the extra $T$ identity-padded dimensions. This is crucial in order to implement the shifting mechanism explained in the following step.

\paragraph{Attention mechanism} The attention mechanism used in this construction 
leverages the fact that transformers process tokens in parallel to implement the recursive parallel scan algorithm. We now detail its parameters and give some intuition as to their use in the construction.

Every self-attention block in this construction has a total of $h=2$ heads. We index the heads using superscripts $(L)$ and $(R)$. This simplifies the notation as the left and right heads process the left and right parts of the embedding respectively. The construction uses $L = \log_2(T)$ layers, where a layer is taken to be the composition$f_\text{mlp}\circ f_\text{attn}$.

For all $l$ layers with $1 \leq l \leq L$ let
\begin{align*}
    \mathbf{W}_Q^{(L)} &= \mathbf{W}_Q^{(R)} = \mathbf{W}_K^{(R)} = 
    \begin{pmatrix}
        \mat{0}_{n^2 \times 2} & \mat{0}_{n^2 \times 2} & \mat{I}_2
    \end{pmatrix}^\top\\
    \mathbf{W}_K^{(L)} &=
    \begin{pmatrix}
        \mat{0}_{n^2 \times 2} & \mat{0}_{n^2 \times 2} & \boldsymbol{\rho}_\theta
    \end{pmatrix}^\top
    \\
    \intertext{where $\boldsymbol{\rho}_\theta$ is the rotation matrix given by}
    \boldsymbol{\boldsymbol{\rho}}_\theta &= \begin{pmatrix}
        \cos \theta & \sin \theta \\
        -\sin \theta & \cos \theta
    \end{pmatrix},
    \intertext{and $\theta$ is defined as}
    \theta &= -\frac{\pi 2^{l-1}}{T}.
\end{align*}
The three first matrices directly select the positional embeddings from the input, and the last matrix selects the positional embeddings and rotates them according to the value of $l$. In essence, this rotation is what creates the power of two shifting necessary for the prefix sum algorithm.

Finally, we set
\begin{align*}
    \mat{W}_V^{(L)} &=
    \begin{pmatrix}
        \mat{I}_{n^2} & \mat{0}_{n^2} & \mat{0}_{{n^2} \times 2}\\
        \mat{0}_{n^2} & \mat{0}_{n^2} & \mat{0}_{{n^2} \times 2}\\
        \mat{0}_{n^2} & \mat{0}_{n^2} & \mat{0}_{{n^2} \times 2}\\
    \end{pmatrix}\\
    \mat{W}_V^{(R)} &=
    \begin{pmatrix}
        \mat{0}_{n^2} & \mat{0}_{n^2} & \mat{0}_{{n^2} \times 2}\\
        \mat{0}_{n^2} & \mat{I}_{n^2} & \mat{0}_{{n^2} \times 2}\\
        \mat{0}_{n^2} & \mat{0}_{n^2} & \mat{0}_{{n^2} \times 2}\\
    \end{pmatrix}\\
\end{align*}
These can be thought of as selector matrices. They select the submatrices corresponding to the left and right embedding sequence blocks.

\paragraph{Feedforward network} The feedforward network in this construction is the same for all $L$ layers. It uses a bilinear layer to compute the vectorized matrix product between transition maps. Note that the bilinear layer is applied batch-wise. This means that the bilinear layer is applied independently to each row of the input matrix.

We define two linear transformations
\begin{align*}
    \mat{W}_{\text{sel}}^{(L)} &= 
    \begin{pmatrix}
        \mat{I}_{n^2} & \mat{0}_{n^2} & \mat{0}_{{n^2} \times 2}
    \end{pmatrix}^\top\\
    \mat{W}_{\text{sel}}^{(R)} &= 
    \begin{pmatrix}
        \mat{0}_{n^2} & \mat{I}_{n^2} & \mat{0}_{{n^2} \times 2}
    \end{pmatrix}^\top.
\end{align*}
These matrices select the left and right embeddings. Next, using Lemma \ref{matmul_lemma}, we define the weight tensor $\ten{T}$. This tensor computes the matrix product between the left and right embeddings of the transition maps previously selected by the linear transformations

Finally, we need one last linear layer to copy the result of the compositions into both the left and the right embeddings. 
To do so, we define 
\begin{align*}
    \mat{W}_\text{out} = \begin{pmatrix}
        \mat{I}_{n^2} & \mat{I}_{n^2} & \mat{0}_{{n^2} \times 2}
    \end{pmatrix}^\top.
\end{align*}

Using this construction, we are able to recover, at the final layer the transition maps $\text{vec}(\mat{A}^\sigma_1), \text{vec}(\mat{A}^{\sigma_1\sigma_2}), ..., \text{vec}(\mat{A}^{\sigma_1...\sigma_T})$ in the right embedding dimension of the output. In order to transform these transition maps into states, we need only apply a transformation $\mat{W}_\text{readout}$ that maps each transition map to its corresponding state such that $\text{vec}(\mat{A}^{\sigma_1...\sigma_t}) \mapsto \boldsymbol{\alpha}^\top\mat{A}^{\sigma_1...\sigma_t}$.

\end{proof}

\begin{lemma}
    Let $\mat{A} \in \R^{n \times n}, \mat{B} \in \R^{n \times n}$ be two square matrices and let $\text{vec}(\cdot)$ denote their vectorization. Then there exists a tensor $\ten{T} \in \R^{n \times n \times n} $ that computes the vectorized matrix product such that
    \begin{align*}
        \ten{T}\times_1\text{vec}(\mat{A})\times_2\text{vec}(\mat{B}) = \text{vec}(\mat{A}\mat{B})
    \end{align*}
    \label{matmul_lemma}
\end{lemma}
\begin{proof}
Let $\mat{C} = \mat{A}\mat{B}$. Using the component wise definition of matrix product, we have
\begin{align*}
    \mat{C}_{ij} &= \sum_{k=1}^{n} \mat{A}_{ik}\mat{B}_{kj}\\
    \intertext{Summing over all possible values in $\mat{A}$ and $\mat{B}$, we can write this as}
    &=\sum_{i',j',i'',j''} \mathbb{1}\{ i'=i, j'=i'', j''=j\}\mat{A}_{i'j'}\mat{B}_{i''j''}\\
    \intertext{Which lets us define the tensor $\ten{T} \in \R^{n^2 \times n^2 \times n^2}$ componentwise as}
    \ten{T}(i',j',j'') &= \mathbb{1}\{ i'=i, j'=i'', j''=j\}.
\end{align*}
This operation can be thought of as taking all multiplicative interactions between the two transition matrices and summing together those that are related through the matrix product.
\end{proof}

\subsection{Proof of Theorem \ref{thmapprox}}
First, let us recall the theorem.
\begin{reptheorem}{thmapprox}
Transformers can \textit{approximately} simulate all WFAs $\mathcal{A}$ at length $T$, up to arbitrary precision $\epsilon > 0$, with depth $\log T$, embedding dimension $2n^2 + 2$, attention width $2n^2 + 2$ and  MLP width $2n^4 + 3n^2 + 1$, where $n$ is the number of states of $\mathcal{A}$.
\end{reptheorem}
The proof of this construction is very similar to that of Theorem \ref{thmexact}, as we use the same key idea of the shifting mechanism. The main difference comes from the fact that we can no longer compute the attention filter exactly and that the MLP is no longer able to compute the composition of transition maps exactly. The most important aspect of this proof is understanding how the error propagates through the construction and choosing sufficient error tolerances. Given a certain $\epsilon$, we want to make sure there exists a construction whose total error is no more than said $\epsilon$. For the MLP, we leverage the result from \citep{chong2020closer} on approximating polynomial functions using neural networks. To state the theorem, we first give some related definitions.

Let $\mathcal{C}(\R)$ denote the set of all real valued functions, $\mathcal{P}_{\leq d}(\R)$ denote the set of all real polynomials of degree at most $d$ and $\mathcal{P}_{\leq d}(\R^{m_1}, \R^{m_2})$ denote the set of all multivariate polynomials from $\R^{m_1}$ to $\R^{m_2}$.
   \begin{theorem}\label{thm_mlp}
      (\textit{abridged version of Theorem 3.1 of \citep{chong2020closer}}) Let $d \geq 2$ be an integer, let $f \in \mathcal{P}_{\leq d}(\R^{m_1}, \R^{m_2})$ and let $\rho_\Theta^\sigma$ be a two-layer MLP with activation function $\sigma$ and parameters $\Theta = (\mat{W}_1, \mat{W}_2)$. If $\sigma \in \mathcal{C}(\R)\setminus \mathcal{P}_{\leq d-1}$, then for every $\epsilon > 0$, there exists some $\Theta \in \{ (\mat{W}_1, \mat{W}_2) \mid \mat{W}_1 \in \R^{m_1 \times N},  \mat{W}_2 \in \R^{N \times m_2}\}$ with $N = {m_1+d \choose d}$ such that $\| f - \rho_\Theta^\sigma\|_\infty < \epsilon$.
   \end{theorem}
Note here that $m,n$ and $ N$ do not depend on $\epsilon$. This means that the size of the construction stays constant for all $\epsilon > 0$.
\begin{proof}
   We prove our result by construction. Given the similarities to the proof of the exact case, we will only detail the sections where the construction differs. Let $\epsilon^* > 0$ be the error tolerance of the transformer.

   \paragraph{Important assumptions}
   The assumptions for this proof are the same as the previous one.

   \paragraph{Embeddings}
   We use exactly the same embedding scheme as in the previous proof.
   
   \paragraph{Attention mechanism}
   For the attention mechanism, we use a similar construction as in the previous proof with some key differences we highlight in this section.

   For this attention mechanism, we still use $h=2$ heads indexed with $(L)$ and $(R)$. We also use $L = \log_2(T)$ layers. 

   For all $l$ layers, with $1 \leq l \leq L$, let:
\begin{align*}
    \mathbf{W}_Q^{(L)} &= \mathbf{W}_Q^{(R)} = \mathbf{W}_K^{(R)} = 
    \sqrt{C}
    \begin{pmatrix}
        \mat{0}_{n^2 \times 2} & \mat{0}_{n^2 \times 2} & \mat{I}_2
    \end{pmatrix}^\top\\
    \mathbf{W}_K^{(L)} &=
    \sqrt{C}
    \begin{pmatrix}
        \mat{0}_{n^2 \times 2} & \mat{0}_{n^2 \times 2} & \boldsymbol{\rho}_\theta
    \end{pmatrix}^\top
    \\
    \intertext{with }
    \boldsymbol{\boldsymbol{\rho}}_\theta &= \begin{pmatrix}
        \cos \theta & \sin \theta \\
        -\sin \theta & \cos \theta
    \end{pmatrix},\\
    \theta &= -\frac{\pi 2^{l-1}}{T}\\
    \intertext{and}
    C &> 0,
\end{align*}
where $C$ is a saturating constant used to approximate hard attention. This lets us approximate the shifting mechanism to arbitrary precision. Notice that $C$ is a constant and thus has no effect on the number of parameters of our construction.

   \paragraph{Feedforward network}
   We do not give an explicit construction for our MLP. Instead, we invoke Theorem \ref{thm_mlp} leveraging the fact that matrix multiplication is a multivariate polynomial. Let $m_1, m_2$ be the input and output dimensions respectively and $N$ be the size of the hidden layer.
   
   First, we set $d = 2$ as the multivariate polynomial corresponding to matrix multiplication considers second order interactions at most. Then, we set $m_1 = 2n^2 + 2$ and $m_2 = n^2$. The input dimension is the embedding dimension and the output dimension is the size of the resulting matrix multiplication.
   
   Using Theorem \ref{thm_mlp}, this gives us a hidden size of $N = {n^2+2\choose 2} = 2n^4 + 3n^2 + 1\in \mathcal{O}(n^4)$. Note that this value does not depend on $\epsilon$.

   Finally, to this MLP, we append the following linear layer:
   \begin{align*}
    \mat{W}_\text{out} = \begin{pmatrix}
        \mat{I}_{n^2} & \mat{I}_{n^2} & \mat{0}_{{n^2} \times 2}
    \end{pmatrix}^\top,
    \end{align*}
    which lets us copy the result of the composition into both the left and right embeddings.

   \paragraph{Error analysis}
   Using an MLP with approximation error $\epsilon_\text{mlp}$ and an attention layer with saturating constant $C$, we want to derive an expression which recursively bounds the error. Let $\mat{A}(\mat{X})$ be the attention filter. After the attention block of the first layer, we have
   \begin{align*}
        \left(\mat{A}(\mat{X}) + \mat{E}_\text{attn}\right)\mat{X}\mat{W}_V &= 
        \mat{A}(\mat{X})\mat{X}\mat{W}_V + \underbrace{\mat{E}_\text{attn}\mat{X}\mat{W}_V}_{=\mat{E}'_\text{attn}}
        \intertext{where $\mat{E}_\text{attn}$ is a matrix representing the error generated by the soft attention. Row-wise at the output of the MLP block, we have}
        (\mat{A}^{\sigma_1} + \text{mat}(\vec{e}'_1))
        (\mat{A}^{\sigma_2} + \text{mat}(\vec{e}'_2)) + \text{mat}(\vec{e}_{i,\text{mlp}}) &=
        \mat{A}^{\sigma_1}\mat{A}^{\sigma_2} +\\
         & \quad \underbrace{\mat{A}^{\sigma_1}\text{mat}(\vec{e}'_2) + \text{mat}(\vec{e}'_1)\mat{A}^{\sigma_2} +\text{mat}(\vec{e}'_1)\text{mat}(\vec{e}'_2) + (\vec{e}_{i,\text{mlp}})}_{\text{error after 1st layer}}
        \intertext{where $\vec{e}'_i$ is the $i$th row vector of $\mat{E}'_\text{attn}$,  $\vec{e}_{i,\text{mlp}}$ is the error generated by the MLP and 
        mat$(\cdot)$ represents matricization. Taking the norm, we get}
        \epsilon^{(1)}_\text{tot} &= \| \mat{A}^{\sigma_1}\text{mat}(\vec{e}'_2) + \text{mat}(\vec{e}'_1)\mat{A}^{\sigma_2} +\text{mat}(\vec{e}'_1)\text{mat}(\vec{e}'_2) + (\vec{e}_{i,\text{mlp}}) \|_2.
        \intertext{Using the triangle inequality and Cauchy-Schwarz, we obtain the following bound}
        \epsilon^{(1)}_\text{tot} &\leq
        \epsilon_\text{attn}\big(\|\mat{A}^{\sigma_1} \|_2 + \| \mat{A}^{\sigma_2}\|_2\big) + \epsilon_\text{attn}^2 + \epsilon_\text{mlp}\\
        &\leq
        \epsilon_\text{attn}\big(\underbrace{2\max_{\sigma \in \Sigma} \|\mat{A}^\sigma \|_2}_{=M}\big) + \epsilon_\text{attn}^2 + \epsilon_\text{mlp}        \intertext{where $\epsilon_\text{attn}$ is the norm of the attention layers error and  $\epsilon_\text{mlp}$ is the norm of the error incurred by the MLP. Thus at the second layer attention mechanism, we get}
         \left(\underbrace{\mat{A}(\mat{X} + \mat{E}^{(1)}_\text{tot})}_{=\mat{A}(\mat{X})} +  \mat{E}_\text{attn}\right)\left(\mat{X}+ \mat{E}^{(1)}_\text{tot}\right)\mat{W}_V &= \mat{A}(\mat{X})\mat{X}\mat{W}_V + 
         \underbrace{\mat{A}(\mat{X})\mat{E}^{(1)}_\text{tot}\mat{W}_V  + \mat{E}_\text{attn}\mat{X}\mat{W}_V + \mat{E}_\text{attn}\mat{E}^{(1)}_\text{tot}\mat{W}_V}_{\text{error after attention layer 2}}.
         \intertext{Here, $\mat{E}^{(1)}_\text{tot}$ is a matrix such that the norm of each row is $\epsilon^{(1)}_\text{tot}$. To simplify the error analysis through the layer 2 MLP, we set}
         \mat{E}^{(1)}_\text{tot} &:= \mat{E}^{(1)}_\text{tot} + \mat{A}(\mat{X})\mat{E}^{(1)}_\text{tot}\mat{W}_V  + \mat{E}_\text{attn}\mat{X}\mat{W}_V + \mat{E}_\text{attn}\mat{E}^{(1)}_\text{tot}\mat{W}_V.
         \intertext{Using a similar approach as for the first layer, we get}
          \epsilon_\text{tot}^{(2)} &\leq \epsilon_\text{tot}^{(1)}M + \left(\epsilon_\text{tot}^{(1)}\right)^2 + \epsilon_\text{mlp}^{(2)}.
   \end{align*}
   
   Thus, we can derive the following recursive expression for $\ell \in [L]$:

    \begin{align}
        \epsilon_\text{tot}^{(\ell)} \leq \epsilon_\text{tot}^{(\ell-1)}M + \left(\epsilon_\text{tot}^{(\ell-1)}\right)^2 + \epsilon_\text{mlp}^{(\ell)},
    \end{align}
    where $M = 2\max_{\sigma \in \Sigma} \| \mat{A}^\sigma\|_2$, $\epsilon_\text{mlp}^{(\ell)}$ is the error incurred by the MLP at layer $\ell$ and $\epsilon_\text{tot}^{(\ell)}$ is the total error at layer $\ell$. For any number of layers, the error remains bounded. This means that we can always choose a large enough $C$ and a small enough $\epsilon^{(\ell)}$ such that $\epsilon_\text{tot}^{(\ell)} \leq \epsilon^*$.
    
    Moreover, given that the size of the MLP $N$ does not depend on the target accuracy epsilon to approximate matrix products, and that the saturating constant $C$ does not affect the parameter count of the attention layer, we get this bound on $\epsilon_\text{tot}^{(\ell)}$ without ever increasing the number of parameters in our construction.

\end{proof}

\section{Proof of Theorem~\ref{thm:approx.wta}}\label{Appendix:C}
First, let us recall the theorem 
\begin{reptheorem}{thm:approx.wta}
Transformers can \textit{approximately} simulate all WTAs $\mathcal{A}$ with $n$ states at length $T$, up to arbitrary precision $\epsilon > 0$, with embedding dimension $\mathcal{O}(n)$, attention width $\mathcal{O}(n)$,  MLP width $\mathcal{O}(n^3)$  and $\mathcal O(1)$ attention heads. Simulation over arbitrary trees can be done with depth $\mathcal{O}(T)$ and simulation over balanced trees~(trees whose depth is of order $\log(T)$) with depth $\mathcal{O}(\log(T))$.    
\end{reptheorem}
Let $\mathcal{A} = \langle \vecs{\alpha} \in \Rbb^n, \ten{T}\in \Rbb^{n\times n \times n}, \{ \vec{v}_\sigma \in \Rbb^n  \}_{\sigma \in \Sigma} \rangle $ be a WTA with $n$ states on $\trees_{\Sigma}$. We will construct a transformer $f$ such that, for any tree $t\in\trees_\Sigma$, the output after $\mathcal{O}(\depth(t))$ layers is such that  $f(\str(t))_i = \mu(\tau(i))$ for all $i\in \mathcal{I}_t$, where $T = |\str(t)|$ and the subtree $\tau(i)$ is defined in Def.~\ref{def:subtree.str.rep}. This will show both parts of the theorem as the depth of any tree $t$ is upper bounded by $|\str(t)|=T$.

The construction has two parts. The first part complements the initial embeddings  with relevant  structural information~(such as the depth of each node). This first part has a constant number of layers. The second part of the transformer computes the sub-tree embeddings  $\mu(\tau_i)$ iteratively, starting from the deepest sub-trees up to the root. After $\depth(t)$ layers of the second part of the transformer, all the subtree embeddings have been computed. 

\begin{proof}
We prove our result by construction. Each section details a specific part of the considered construction.
\paragraph{Initial embedding}
The initial embedding for the $i$th symbol $\sigma_i$ in $\str(t)$ is given by
$$ \vec x^{(0)}_i = ( \vec v_{\sigma_i} \concat \vec p_i \concat m_i \concat 1)$$
where
\begin{itemize}
    \item $\concat$ denotes vector concatenation
    \item  $\vec v_{\sigma_i}$ is taken from the WTA $\mathcal A$ if $\sigma_i\in\Sigma$ and $\vec v_{\sigma_i} = \vec 0$ if $\sigma_i \in \{\opensymbol,\closesymbol\}$
    \item $\vec p_i 
    $ is the positional encoding
    \item $m_i$ is a marker to distinguish leaf symbols, opening and closing parenthesis. It is defined by $m_i = 0$ if $\sigma_i \in \Sigma$, $m_i = 1$ if $\sigma_i = \opensymbol$ and  $m_i = -1$ if $\sigma_i = \closesymbol$.
    \item the last entry equal to $1$ is for convenience (to compactly integrate the bias terms in the attention computations).
\end{itemize}

\paragraph{Enriched embedding}
The purpose of the first layers of the transformer is to add structural information related to the tree structure to the initial embedding. We want to obtain the following representation for the $i$th symbol $\sigma_i$ in $\str(t)$:
$$ \vec x^{(1)}_i = ( \vec v_{\sigma_i} \concat \vec p_i \concat m_i \concat 1\concat d_i \concat d_i^2 )$$
where $d_i$ is the depth of the root of $\tau_i$ in $t$.  

Computing the depth at each position $i\in\mathcal{I}_t$ can easily be done with two layers. The role of the first layer is simply to add a new component corresponding to shifting the markers $m_i$ one position to the right~(as shown in the construction for simulating WFAs, this can be done easily by the attention mechanism). After this first layer, we will have the intermediate embedding
$$ \vec x^{(0.5)}_i = ( \vec v_{\sigma_i} \concat \vec p_i \concat m_i \concat 1\concat m_{i+1} ).$$

Now one can easily check that, by construction, $\sum_{j\leq i} m_{j+1} = \depth(\tau_i)$ for all $i\in \mathcal{I}_t$ since the summation is equal to the difference between the number of open and closed parenthesis before position $i$. Hence, the  attention mechanism of the second layer can compute the sum $d_i = \sum_{j\leq i} m_{j+1}$ by attending to all previous positions with equal weight. The component $d_i^2$ can then be approximated to arbitrary precision by an MLP layer with a constant number of neurons~(see Theorem~\ref{thm_mlp}).

\newcommand{\pos}{\mathrm{pos}}
\paragraph{Tree parsing} The next layers of the transformer are used to compute the final output of the WTA, $\mu(t)$. The two heads of each attention layer are built such that each position $i \in \pos(\opensymbol) := \{i \mid \sigma_i = \opensymbol \}$ attends to the corresponding left and right subtrees, respectively. The MLP layers are used to compute the bilinear map $(\mu(\tau),\mu(\tau^\prime))\mapsto \ten T \times_1 \mu(\tau) \times_2 \mu(\tau') $, which can be approximated to an arbitrary precision
~(from Theorem~\ref{thm_mlp}).

$-$ \textbf{Left head} $-$  First observe that, by construction, for each $i\in\mathcal{I}_t$, the left child of $\tau_i$ is $\tau_{i+1}$. We thus let the left head attention weight matrices $\mat W_Q^{(L)}, \mat W_K^{(L)} \in \R^{(n+6)\times 2}$ be defined by
$$A^{(L)}_{i,j} = \mat x_i^\top \mat W_Q^{(L)} \mat W_K^{(L)}{}^\top \mat x_j  =  \vec p_i^\top \mat R^\top \vec p_j$$
where $\mat R$ denotes the matrix of a 2D rotation of angle $\frac{\pi}{T}$. 
One can easily check that $i+1=\argmax_j A_{i,j}$, thus by multiplying the weight matrices by a large enough constant, the attention mechanism will have each position attend to the next one. We then use the value matrix $\mat W_V^{(L)}$ to select the first part of the corresponding embedding. The output of the left head is thus given by
$$\mat H^{(L)} = (\vec v_{\sigma_2}, \vec v_{\sigma_3},\cdots, \vec v_{\sigma_T},\vec v_{\sigma_T}) $$
and satisfies $\mat H^{(L)}_{:,i} = \vec v_{\mathrm{left}(i)}$ for all $i\in \mathcal{I}_t$, where $\mathrm{left}(i)$ denotes the index of the left child of $\tau(i)$ in $\str(t)$.

$-$ \textbf{Right head} $-$ 
As mentioned above, for each position $i\in\pos(\opensymbol)$, the left child of $\tau_i$ is $\tau_{i+1}$, which is also the next tree in the sequence $\tau_i,\tau_{i+1},\cdots,\tau_T$ whose depth is $\depth(\tau_i)+1$. Similarly, the right child of $\tau_i$ is the second tree in the sequence with depth equal to $\depth(\tau_i)+1$. Thus, we use the attention mechanism to have position $i$ attend to the closest position $j>i+1$ satisfying $d_j = d_i + 1$. In order to do so, we let the right head attention weight matrices $\mat W_Q^{(R)}, \mat W_K^{(R)} \in \R^{(n+6)\times 7}$ be such that
$$A^{(R)}_{i,j} = \mat x_i^\top \mat W_Q^{(R)} \mat W_K^{(R)}{}^\top \mat x_j  = - \beta (1 - (d_j-d_i))^2 + \vec p_i^\top \mat R^\top \vec p_j + 2 \mathbb{I}[j \geq i+2]$$
where $\mat R$ denotes the matrix of a 2D rotation of angle $\frac{2\pi}{T}$. 
The first term ensures that  $j^* = \argmax_j A_{i,j} $ is such that $d_{j^*} = d_i+1$; we choose $\beta$ to be a constant large enough such that the attention weights $A_{i,j}$ are very small for all $j$ such that $d_{j} \neq d_i+1$~(while they are unilaterally $0$ for all positions such that $d_j=d_i+1$). For all positions $j$ such that $d_{j} = d_i+1$, the second term enforces that the closest one to position $i+2$ is chosen. Lastly, the last term enforces that $j^* \geq i+2$. It is obtained by using the Fourier approximation of the Heaviside step function which can be constructed using a constant number of positional embeddings and feedforward layers:
$$\mathbb{I}[j \geq i+2] = \frac{1}{2}\left(1+\sum_{l=0}^k \frac{1}{2l+1}\sin\left((2l+1)(i-j-\frac{1.25 \pi}{T})\right)\right) \simeq 
\begin{cases}
    1 & \text{ for } j = i+2,i+3, ..., T \\
    0 & \text{ for } j = 1, \cdots, i+1
\end{cases}
$$
We thus have that, for all $i\in\mathcal{I}_t$, the position with largest attention weight, $j^*=\argmax_j A_{i,j}$, is equal to the second position after $i$ satisfying $d_j = d_{i+1}$, which is the position of the right subtree of $\tau(i)$. By multiplying the weight matrices by a large enough constant, the attention mechanism will thus have each position in $\mathcal{I}_t$ attend to the corresponding right subtree. We then use the value matrix $\mat W_V^{(R)}$ to select the first part of the corresponding embedding. The output of the right head $\mat H^{(R)}$  thus satisfies $\mat H^{(R)}_{:,i} = \vec v_{\mathrm{right}(i)}$ for all $i\in \mathcal{I}_t$, where $\mathrm{right}(i)$ denotes the index of the right child of $\tau(i)$ in $\str(t)$.

$-$ \textbf{Computing embeddings of depth $1$ subtrees} $-$ We use a third attention layer to simply copy the input tokens. The MLP is thus fed the input vectors 
$$
\tilde{\vec x}_i = (\vec x_i^{(\mathrm{left})} \concat \vec x_i^{(\mathrm{right})}  \concat  \vec v_{\sigma_i} \concat \vec p_i \concat m_i \concat 1\concat d_i \concat d_i^2)
$$ 
in a batch. These inputs are constructed such that, for all positions $i\in\mathcal{I}_t$,   $\vec x_i^{(\mathrm{left})} =  \vec v_{\mathrm{left}(i)}$ and  $\vec x_i^{(\mathrm{right})} =   \vec v_{\mathrm{right}(i)}$. We thus choose the weight of the MLP layer such that it approximates the map
$$ \tilde{\vec x}_i  \mapsto \left( m_i^2 \cdot (\ten T \ttm{1} \vec x_i^{(\mathrm{left})} \ttm{2} \vec x_i^{(\mathrm{right})}) + (1-m_i^2) \vec v_{\sigma_i} \concat \vec p_i \concat m_i \concat 1\concat d_i\concat d_i^2\right)$$ 
to an arbitrary precision. Since this map is a $4$th order polynomial, this can be done with arbitrary precision with $\mathcal{O}(n^4)$ neurons~(from Theorem~\ref{thm_mlp}).

First, observe that we only care about the indices in $\mathcal{I}_t$, which correspond to leaf symbols or opening parenthesis. We now make the following observations: 
\begin{itemize}
    \item For positions corresponding to leaf symbols~(\textit{i.e.} all sub-trees of depth $0$), we have $m_i = 0$, thus this first attention layer copies only the corresponding input token without any modifications, which already contains the sub-tree embedding $\mu(\sigma_i)$ since the first embedding layer.
    \item Similarly, for all positions $i\in\mathcal{I}_t$ such that $\depth (\tau(i)) > 1 $, this first layer only copies the corresponding input. Indeed, for such positions we necessarily have that (i) $\sigma_i=\opensymbol$, thus $m_i=1$ and $\vec v_{\sigma_i}=\vec 0$ and (ii) at least one of the children of $\tau(i)$ is not a leaf and has an initial embedding equal to zero, hence $\ten T \ttm{1} \vec x_i^{(\mathrm{left})} \ttm{2} \vec x_i^{(\mathrm{right})}=\vec 0$.
    \item For all positions $i\in\mathcal{I}_t$ such that $\depth (\tau(i)) = 1 $, we have that both child embeddings $\vec x_i^{(\mathrm{left})}$ and $  \vec x_i^{(\mathrm{right})}$ have been initialized to the corresponding leaf embeddings $\mu(\tau(\mathrm{left} (i) ))$ and $\mu(\tau(\mathrm{right} (i) ))$, respectively. Hence, for such positions, the corresponding output tokens are equal to   
    \begin{align*}
    \tilde{\vec x}_i 
    &= 
    \left(  \ten T \ttm{1} \mu(\tau(\mathrm{left} (i) )) \ttm{2} \mu(\tau(\mathrm{right} (i) )) \concat \vec p_i \concat m_i \concat 1\concat d_i\concat d_i^2\right)\\
    &=
    \left(  \mu(\tau_i) \concat \vec p_i \concat m_i \concat 1\concat d_i\concat d_i^2\right) .
    \end{align*}
\end{itemize}

It follows that after this transformer layer, the output tokens $\vec x^{(2)}_1, \cdots, \vec x^{(2)}_T  $ are such that, for any $i\in\mathcal{I}_t$ such that $\depth(\tau(i)) \leq 1$, we have $x^{(2)}_i = \mu(\tau_i) $.  

$-$ \textbf{Computing embeddings of all subtrees} $-$ One can then check that by constructing the following layers in a similar fashion, the output tokens $\vec x^{(\ell)}_1, \cdots, \vec x^{(\ell)}_T  $ will be such that $x^{(\ell)}_i = \mu(\tau_i) $ for any $i\in\mathcal{I}_t$ satisfying $\depth(\tau(i)) \leq \ell - 1$, which concludes the proof.

\end{proof}

\section{Experiments}
\subsection{Experimental Details}
In this section, we give an in-depth description of the training procedure used for the experiments in Section \ref{experiments}.

\paragraph{General considerations}
For all experiments, we use the PyTorch TransformerEncoder implementation and use a model with 2 attention heads.
We train using the AdamW optimizer with a learning rate of $0.001$ as well as MSE loss with mean reduction. We use a standard machine learning pipeline with an 80, 10, 10 train/validation/test split and retain the model with best validation MSE for evaluation on the test set. We evaluate our models on a sequence to sequence task, where, for a given input sequence, the transformer must produce as output the corresponding sequence of states. All experiments are conducted on synthetic data with number of examples $N = 10\;000$. For each task, we record the mean and minimum MSE over 10 runs. All experiments were run on the internal compute cluster of our institution.

\paragraph{Experiments with Pautomac}
For the experiments using the automata from the Pautomac \cite{pautomac} dataset, we consider only hidden Markov models (HMMs) and probabilistic finite automata (PFA), as deterministic probabilistic finite automata (DPFAs) are very close to DFAs and are more in the scope of the results of \citep{liu2022transformers}. We also consider only automata with a number of states inferior to 20 to keep the size of the required transformers small and use a hidden layer size/embedding size of 64 for all experiments. We use a linear layer followed by softmax at the output for readout, as the task at hand implies computing probability distributions. For all experiments, we use synthetic data sampled uniformly from the automata's support with sequence length $T=64$.

\paragraph{Experiments with counting WFA}
For the experiments using the WFA which counts 0s, we use an embedding size and a hidden layer size of 16 and use a linear layer for readout at output. Note that here we \textbf{do not} append a softmax layer to the linear layer. We consider sequence lengths $T \in \{ 16, 32, 64\}$ and number of layers $L \in \{ 1, 2, \hdots, 10\}$. The synthetic data is generated using the following procedure:

For each $t \in [T]$
\begin{itemize}
    \item Sample a sequence $x$ uniformly from $\Sigma = \{0, 1\}$.
    \item Compute the sequence of states for $x$ and store in a $T \times n$ array.
\end{itemize}
An interesting remark concerning this experiment is that if we round the output to the nearest integer at test time, we obtain an MSE of 0.

\paragraph{Experiments with $k$-counting WFA}
For the experiments with the $k$-counting WFA, we fix the number of layers to 4 and evaluate the model on sequences of length $T=32$. Here, we consider $k \in \{2,4,6,8 \}$ and choose the embedding size $d \in \{2, 4, 8, 16, 32, 64 \}$. Note here that we use the same value for both embedding dimension and hidden size. As it is the case for the binary counting WFA, here we also use a linear layer for readout. The data is generated using the same procedure as described in the above section with the exception that here we sample from $\Sigma = \{ 0, 1, \hdots, k-1\}$.

\subsection{Additional Experiments}
In this section, we present an extended version of the results of the experiments presented in Section \ref{experiments}.
\subsubsection{Can logarithmic solutions be found?}
Here, we present the full table of results containing all MSE values for each considered automata. Here we report the minimum MSE over 10 runs and bold the best MSE for each automaton. 

 \begin{table}[H]
   \caption{MSE for all Pautomac automata}
   \label{sample-table}
   \centering
   \begin{tabular}{llllll}
     \toprule
     Nb/Nb layers & 2 & 4 & 6 & 8 & 10 \\
     \midrule
pautomac 12 & 0.005486 & 0.001660 & 0.000770 & \textbf{0.000356} & 0.000710 \\ 
pautomac 14 & 0.000264 & 0.000130 & \textbf{0.000109} & 0.000158 & 0.006189 \\
pautomac 20 & 0.007433 & 0.002939 & 0.000911 & \textbf{0.000628} & 0.000979 \\
pautomac 30 & 0.029165 & 0.017486 & 0.013498 & \textbf{0.012403} & 0.068889 \\
pautomac 31 & 0.007002 & 0.003804 & 0.001114 & \textbf{0.000890} & 0.000899 \\
pautomac 33 & 0.003654 & \textbf{0.001160} & 0.008794 & 0.017104 & 0.016844 \\
pautomac 38 & 0.001056 & 0.000466 & 0.000316 & 0.000216 & \textbf{0.000213} \\
pautomac 39 & 0.014218 & 0.002677 & \textbf{0.001310} & 0.002736 & 0.002686 \\
pautomac 45 & 0.020730 & \textbf{0.018859} & 0.021852 & 0.024375 & 0.023893 \\
     \bottomrule
   \end{tabular}
 \end{table}
 
\subsubsection{Do solutions scale as theory suggests?}
Here we present the plots for both the mean and minimum MSE values for the synthetic experiments on number of layers and embedding size. We equally include tables containing the average MSE values with their respective standard deviation values. We include these values in a table instead of on the plots directly for readability.
\begin{figure}[H]
    \centering

    \begin{subfigure}{0.45\textwidth}
        \includegraphics[scale=1]{Average_MSE_vs_number_of_layers.pdf}
        \caption{Average MSE over 10 runs}
        \label{fig:layersfig1}
    \end{subfigure}
    \hfill 
    \begin{subfigure}{0.45\textwidth}
        \includegraphics[scale=1]{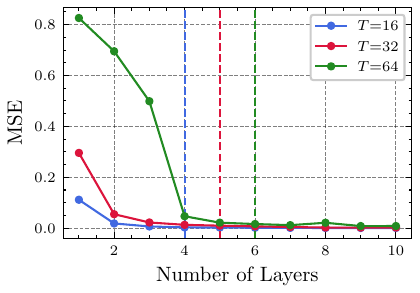}
        \caption{Minimum MSE over 10 runs}
        \label{fig:layersfig2}
    \end{subfigure}

    \caption{MSE vs. number of layers: We notice that for both Figures \ref{fig:layersfig1}} and \ref{fig:layersfig2}, the trend is very similar
    \label{fig:main}
\end{figure}

\begin{figure}[H]
    \centering

    \begin{subfigure}{0.45\textwidth}
        \includegraphics[scale=1]{Average_MSE_vs_embedding_size.pdf}
        \caption{Average MSE over 10 runs}
        \label{fig:embfig1}
    \end{subfigure}
    \hfill 
    \begin{subfigure}{0.45\textwidth}
        \includegraphics[scale=1]{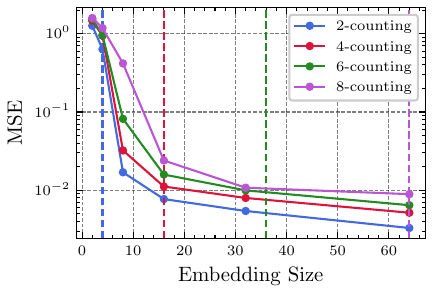}
        \caption{Minimum MSE over 10 runs}
        \label{fig:embfig2}
    \end{subfigure}

    \caption{MSE vs. embedding size: Here the trend between average and minimum values is also very similar}
    \label{fig:main}
\end{figure}

 \begin{table}[H]
   \caption{MSE with standard deviation across layers}
   \label{layers-table}
   \centering
\begin{tabular}{rrrr}
\toprule
number of layers & $L=16$ & $L=32$ & $L=64$\\
\midrule
1 & 0.202724 $\pm$ 0.162529 & 0.445007 $\pm$ 0.337551 & 1.107175 $\pm$ 0.660082 \\
2 & 0.020226 $\pm$ 0.01240 & 0.073415 $\pm$ 0.0454464 & 0.834934 $\pm$ 0.253568\\
3 & 0.008247 $\pm$ 0.002069 & 0.030223 $\pm$ 0.0099422 & 0.718981 $\pm$ 0.506068\\
4 & 0.003796 $\pm$ 0.002531 & 0.016276 $\pm$ 0.0187777 & 0.198496 $\pm$ 0.369018\\
5 & 0.003137 $\pm$  0.001482 & 0.014042 $\pm$ 0.0087545 & 0.077845 $\pm$ 0.1208077\\
6 & 0.002000 $\pm$ 0.001380 & 0.009957 $\pm$ 0.008190 & 0.055315 $\pm$ 0.060973\\
7 & 0.001449 $\pm$ 0.002166 & 0.006567 $\pm$ 0.004517 & 0.049691 $\pm$ 0.09459\\
8 & 0.001223 $\pm$ 0.000897 & 0.004663 $\pm$ 0.006610 & 0.044012 $\pm$ 0.059172\\
9 & 0.000945 $\pm$ 0.000558 & 0.003357 $\pm$ 0.002276 & 0.046805 $\pm$ 0.146278\\
10 & 0.001156 $\pm$ 0.000687 & 0.003093 $\pm$ 0.004135 & 0.029291 $\pm$ 0.059460\\
\bottomrule
\end{tabular}
\end{table}

 \begin{table}[H]
   \caption{MSE with standard deviation across embedding sizes}
   \label{embedding-table}
   \centering
\begin{tabular}{rrrrr}
\toprule
 embedding size & $k=2$ & $k=4$ & $k=6$ & $k=8$ \\
\midrule
 2 & 1.316259 $\pm$ 0.0361747 & 1.537641 $\pm$ 0.0553744 & 1.619518 $\pm$ 0.0454326 & 1.669770 $\pm$ 0.046430\\
 4 & 0.643738 $\pm$ 0.0401225& 1.016194 $\pm$ 0.704216 & 1.021940 $\pm$ 0.091605& 1.208901 $\pm$0.092536\\
 8 & 0.023678 $\pm$ 0.011056& 0.038078 $\pm$ 0.011791 & 0.186437 $\pm$ 0.769370& 0.459510 $\pm$ 0.164538\\
 16 & 0.010134$\pm$ 0.0027882& 0.01299 $\pm$ 0.002661 & 0.01768 $\pm$ 0.003942 & 0.026244 $\pm$ 0.003718\\
32 & 0.006623 $\pm$ 0.002187& 0.008674 $\pm$ 0.002234 & 0.010712 $\pm$ 0.001917& 0.011528 $\pm$ 0.001126\\
64 & 0.003882 $\pm$ 0.001105& 0.006435 $\pm$  0.003151 & 0.007041 $\pm$ 0.001648 & 0.009701 $\pm$ 0.002908\\
\bottomrule
\end{tabular}
\end{table}

\end{document}